\documentclass{article}

\usepackage[preprint,nonatbib]{neurips_2023}

\usepackage[utf8]{inputenc} 
\usepackage[T1]{fontenc}    
\usepackage{hyperref}       
\usepackage{url}            
\usepackage{booktabs}       
\usepackage{amsfonts}       
\usepackage{nicefrac}       
\usepackage{microtype}      
\usepackage{xcolor}         
\usepackage{float}

\usepackage{algorithmic}
\usepackage{graphicx}
\usepackage{textcomp}
\usepackage{epstopdf}
\usepackage{times}
\usepackage{amssymb}
\usepackage{amsthm}
\usepackage{enumitem}
\usepackage{algorithm}

\usepackage{amsmath,amsfonts,bm}

\def\1{\bm{1}}

\def\vmu{{\bm{\mu}}}
\def\vtheta{{\bm{\theta}}}

\def\vw{{\bm{w}}}
\def\vx{{\bm{x}}}

\def\vz{{\bm{z}}}

\def\mB{{\bm{B}}}

\def\mF{{\bm{F}}}

\def\mI{{\bm{I}}}

\def\mM{{\bm{M}}}

\def\mX{{\bm{X}}}
\def\mY{{\bm{Y}}}

\def\mSigma{{\bm{\Sigma}}}

\DeclareMathAlphabet{\mathsfit}{\encodingdefault}{\sfdefault}{m}{sl}
\SetMathAlphabet{\mathsfit}{bold}{\encodingdefault}{\sfdefault}{bx}{n}

\DeclareMathOperator*{\argmax}{arg\,max}
\DeclareMathOperator*{\argmin}{arg\,min}

\newcommand{\nn}{\nonumber}

\newtheorem{lemma}{\textbf{Lemma}}
\newtheorem{theorem}{\textbf{Theorem}}
\newtheorem{assumption}{\textbf{Assumption}}
\newtheorem{proposition}{\textbf{Proposition}}
\newtheorem{remark}{\textbf{Remark}}

\title{Gibbs-Based Information Criteria and the Over-Parameterized Regime}

\author{%
  Haobo Chen \\
  University of Florida\\
  \texttt{haobo.chen@ufl.edu} \\
   \And
   Yuheng Bu\\
   University of Florida\\
   \texttt{buyuheng@ufl.edu} \\
  \And
   Gregory W. Wornell\\
   Massachusetts Institute of Technology\\
   \texttt{gww@mit.edu} \\
}

\begin{document}
\maketitle
%

%

\begin{abstract}
Double-descent refers to the unexpected drop in test loss of a learning algorithm beyond an interpolating threshold with over-parameterization, which is not predicted by information criteria in their classical forms due to the limitations in the standard asymptotic approach. We update these analyses using the information risk minimization framework and provide Akaike Information Criterion (AIC) and Bayesian Information Criterion (BIC) for models learned by the Gibbs algorithm. Notably, the penalty terms for the Gibbs-based AIC and BIC correspond to specific information measures, i.e., symmetrized KL information and KL divergence.
We extend this information-theoretic analysis to over-parameterized models by providing two different Gibbs-based BICs to compute the marginal likelihood of random feature models in the regime where the number of parameters $p$ and the number of samples $n$ tend to infinity, with $p/n$ fixed. Our experiments demonstrate that the Gibbs-based BIC can select the high-dimensional model and reveal the mismatch between marginal likelihood and population risk in the over-parameterized regime, providing new insights to understand double-descent.


\end{abstract}

\section{Introduction}
The classical understanding of model selection is that more complex models can capture more complex patterns but tend to overfit and have large generalization error \cite{geman1992}.
This tradeoff results in a $\cup$-shaped curve which is characterized by the classical model selection criterion when test loss is plotted against model complexity. As a result, the models that minimize test loss tend to have moderate complexity.
Recently, the success of deep learning challenges this classical picture since neural networks are often extremely complex (e.g., able to fit random labels \cite{zhang2016understanding}) while \textit{also} generalizing well to yield low test error on unseen samples.

An emerging explanation of this behavior is \textit{double-descent}  \cite{belkin2019reconciling}, which posits that:
1) The classical $\cup$-shaped curve is only valid when the number of model parameters $p$ is smaller than the number of samples $n$.
2) In the over-parameterized regime where $p$ is significantly larger than $n$, and models are complex enough to fit training data perfectly, test loss can decrease with increased model complexity.


To better understand the double-descent phenomenon, we revisit the classical derivations of information criteria and discern that the penalty term in Akaike Information Criterion (AIC) can be interpreted as the generalization error within a broader learning context, while Bayesian Information Criterion (BIC) approximates the marginal likelihood using the empirical risk minimization solution. We further update the classical analyses of AIC and BIC using the information risk minimization framework proposed in \cite{zhang2006information} by focusing on the optimal Gibbs algorithm (distribution). It is shown in \cite{aminian2021exact} that the generalization error of the Gibbs algorithm can be characterized using information measures. This information-theoretic analysis motivates the proposed Gibbs-based information criteria, which readily extend to over-parameterized models. 

\begin{figure*}[!t]
    \centering
    \includegraphics[width=0.42\columnwidth]{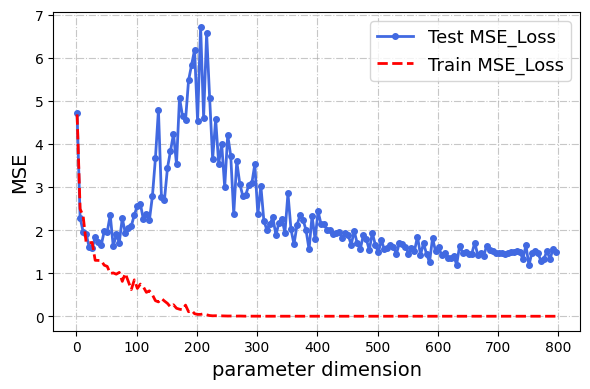}\hspace*{2em}
    \includegraphics[width=0.42\columnwidth]{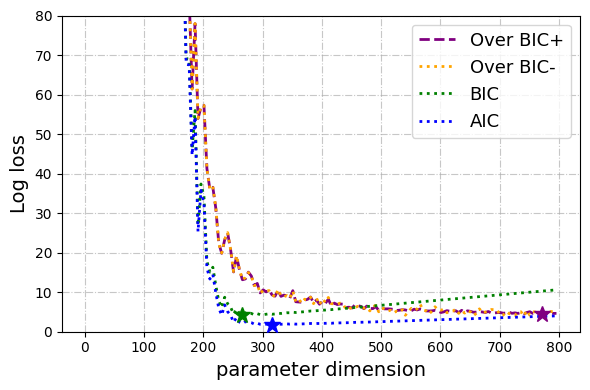}
     \caption{ 
     The test and training mean-squared error (MSE) (\textbf{left}) and the comparison of the proposed over-parameterized $\mathrm{BIC}^+$ and  $\mathrm{BIC}^-$with other classical information criteria (\textbf{right}) for over-parameterized RF models plotted with varying parameter dimension $p$. More details can be found in Section~\ref{sec:exp_over}.}
    \vspace{-1em}
\label{fig:over_loss}
\end{figure*}

We make the following contributions in this paper:
\begin{enumerate}[leftmargin=*,topsep=-0.2em,itemsep=-0.4pt] 
    \item We provide information-theoretic analyses for the generalization error and marginal likelihood of the model learned by the Gibbs algorithm, resulting in Gibbs-based AIC~\eqref{equ: Gibbs_AIC} and Gibbs-based $\mathrm{BIC}^+$~\eqref{equ: Gibbs_BIC} and $\mathrm{BIC}^-$~\eqref{equ: Gibbs_BIC-}, with different information measures as the penalty terms. 
    \item We show that the Gibbs-based information criteria align with the classical information criteria in the classical large $n$ regime theoretically (Theroem~\ref{thm:AIC_classical} and~\ref{thm:BIC_Gibbs_classical}) and empirically. 
    \item We generalize our information-theoretic analysis to over-parameterized random feature (RF) models, which results in over-parameterized Gibbs-based BICs (\eqref{equ:BICover} and~\eqref{equ:BIC-over}) that favor over-parameterized RF model, while classical information criteria cannot; see Figure~\ref{fig:over_loss}.    
    \item We empirically compare the Gibbs-based BICs and AIC in the over-parameterized RF model by decomposing them into different terms (Section~\ref{sec:exp_over}), and show the mismatch between marginal likelihood (BIC) and generalization error (AIC) in the over-parameterized setting, where AIC exhibits double-descent but BIC does not. Such a phenomenon is highly affected by the choice of prior distributions.  
    
\end{enumerate}

\paragraph{Related Work}
Previous work has extended the classical BIC to high dimensions, e.g.,\cite{chen2008extended,fan2013tuning}. However, these works seek to \emph{replace} maximizing marginal likelihood with a different criterion, substituting a penalty term $p f(n)$ in place of the $p \log n$ in BIC. 
A Widely Applicable BIC (WBIC) is proposed in \cite{watanabe2013widely}, which extends the BIC to singular probabilistic models using the asymptotic approximation of the marginal likelihood with Gibbs posterior in the classical large data regime. By contrast, we retain the BIC criterion but analyze it beyond the classical regime for the Gibbs algorithm with an exact information-theoretic analysis.


Double-descent of the population risk with increasing model size was introduced in \cite{belkin2019reconciling}; see also \cite{advani2020high,geiger2019jamming}.
An empirical demonstration of double-descent in modern deep networks is provided in  \cite{nakkiran2019deep}.
A variety of work develops simplified models where the characterization of the double-descent curve can be obtained.
For example, double-descent in linear regression models is investigated in \cite{Belkin_2020,hastie2022surprises,bartlett2020benign,muthukumar2020harmless},
and in linear classification models in \cite{deng2022model,kini2020analytic,gerace2020generalisation}.
The RF model has been adopted to understand double-descent in \cite{mei2022generalization}, which provides a generalization analysis of the performance achieved with ridge regression in the over-parameterized regime.
In some of our analysis we likewise adopt this RF model~\cite{mei2022generalization, d2020double,gerace2020generalisation,weinan2020comparative,liu2022double}, but with a different objective and analysis tools.

Double-descent phenomena have been explained from different perspectives. In
\cite{d2020double, yang2020rethinking,dwivedi2020revisiting}, double-descent curves are explained via a refined version of bias-variance tradeoff, where the bias of the model decreases monotonically with the increase of $p$, but the variance increases and then decreases with $p$.
And the connection of gradient descent dynamics and double-descent is discussed in \cite{weinan2020comparative,advani2020high}. In \cite{nakkiran2019more,d2020triple}, sample-wise double-descent is studied under linear regression, and \cite{heckel2020early} shows that by adjusting the step sizes, sample-wise double-descent can be eliminated by early stopping.


Among recent work, \cite{germain2016pac} connects the PAC-Bayes generalization bound with marginal likelihood and approximates the KL divergence for model selection using MCMC algorithms. However, they did not consider the over-parameterized regime. For over-parameterized models, \cite{hodgkinson2023interpolating} study the Interpolating BIC with a special focus on interpolated models with zero training loss.
Our paper is most related to \cite{immer2021scalable,lotfi2022bayesian}, which also examines the difference between marginal likelihood and generalization error in model selection. However, by focusing on the Gibbs algorithm, we are able to interpret the mismatch between AIC and BIC via information measures, which is more insightful in understanding double descent and other complex behaviors. 




\section{Preliminaries}\label{sec:pre}

We consider the following standard supervised learning formulation. 
Let $S = \{Z_i\}_{i=1}^n$ be the training set, where each sample $Z_i= \{(X_i,Y_i)\} \in \mathcal{Z}$ are i.i.d. generated from the data distribution $P_Z$, and with the realization of this sequence denoted as $\vz^n = (\vx^n,y^n)$.
We denote the parameter of a machine learning model by $w \in \mathcal{W}$, where $\mathcal{W}$ is the parameter space. The performance of the model is measured by a loss $\ell\colon\mathcal{W} \times \mathcal{Z}  \to \mathbb{R}$, and the log loss $\ell_{\log}(w,\vz) \triangleq -\log P(y | \vx;w)$ associated with a parametric probabilistic model $P(y | \vx;w)$ is of particular interest to us.

We define the empirical risk and the population risk associated with a given $w$ as
\begin{align} \label{equ:risks}
    L_{E}(w, \vz^n) &\triangleq \frac{1}{n}\sum_{i=1}^{n}\ell(w,\vz_{i}),  \\
    L_{P}(w,P_{Z}) &\triangleq \mathbb{E}_{P_{Z}}\big[\ell(w,Z)\big],
\end{align}
respectively. A learning algorithm can be modeled as a randomized mapping from the training set $S$ onto a model parameter $\hat{W}$  according to the conditional distribution $P_{\hat{W}|S}$.

Although the widely-used empirical risk minimization (ERM) is deterministic optimization,
it is usually solved via stochastic gradient descent (SGD), which is random in nature. The ERM solution of log loss is the well-known maximum likelihood estimate (MLE), i.e.,  $\hat{W}_{\mathrm{ML}} \triangleq \argmax_{w} \hat{L}(w)$, 
where $\hat{L}(w) \triangleq   \sum_{i=1}^n \log  P_k(y_i| \vx_i ;w)$
denotes the log-likelihood of $\vz^n$. The expected generalization error quantifying the degree of over-fitting can be expressed in the form
\begin{equation}\label{equ:gen}
    \overline{gen}(P_{\hat{W}|S},P_{S}) \triangleq \mathbb{E}_{P_{\hat{W},S}}\big[L_{P}(\hat{W},P_Z)-L_{E}(\hat{W},S)\big],
\end{equation}
where the expectation is taken over the joint distribution $P_{\hat{W},S} =  P_{\hat{W}|S}\otimes P_S$.

\subsection{The Classical Forms of AIC and BIC}
\label{sec:pre_AIC}


The standard derivation of the AIC and the BIC arises from the classical asymptotic analysis of MLE. Assume we have $K$ candidate models $M_1, M_2, \ldots, M_K$, and each model $M_k$ is characterized by a parametric probabilistic model $P_k(y | \vx;\vtheta_k)$,
 a prior distribution $\pi_k(\vtheta_k)$, where
$\vtheta_k \in \mathcal{W}_k \subset \mathbb{R}^{p_k}$ is the parameter vector.
We demonstrate that AIC selects the model with the smallest population risk, and BIC identifies the true data-generating model by maximizing the marginal likelihood. 


\paragraph{AIC}
This criterion \cite{RePEc:eee:econom:v:16:y:1981:i:1:p:3-14} ranks statistical models based on the Kullback-Leibler (KL) divergence between the true data distribution $P_Z$ and the learned parametric model. With $\smash{\hat{\vtheta}_{\mathrm{ML}}^{(k)}}$ denoting the MLE of the $k$\/th model,
AIC selects the model as the solution to
\begin{align}\label{equ:approx_kl_divergence}
&\underset{k}{\operatorname{argmin}} \ D(P_{Z}\|P_k(y| \vx ;\hat{\vtheta}_{\mathrm{ML}}^{(k)}))\\
=&\underset{k}{\operatorname{argmin}}\  \mathbb{E}_{P_{Z}}\big[-\log P_k(y| \vx ;\hat{\vtheta}_{\mathrm{ML}}^{(k)})\big]. 
\end{align}
The term $\mathbb{E}_{P_{Z}}[-\log P_k(y| \vx ;\hat{\vtheta}_{\mathrm{ML}}^{(k)})]$ can be interpreted as the population risk $L_{P}(\hat{\vtheta}_{\mathrm{ML}}^{(k)},P_{Z})$ of the MLE under log-loss. From this perspective, AIC measures how well the model fits the unknown data distribution $P_Z$, with smaller AIC values suggesting a lower population risk. 



As the true distribution $P_Z$ is unknown, the AIC is obtained by approximating the population risk as the sum of empirical risk, i.e., the negative log-likelihood of $\smash{\hat{\vtheta}_{\mathrm{ML}}^{(k)}}$ on training samples and a penalty term corresponding to the generalization error. In the classic regime where $p_k$ is fixed and $n\to \infty$, the
asymptotic normality of MLE yields
\begin{align} \label{equ:AIC}
    \mathrm{AIC} &=-
    \frac{\hat{L}(\hat{\vtheta}_{\mathrm{ML}})}{n} + \frac{p}{n}.
\end{align}
Note that our form of AIC differs by a factor of 2 from its classical form to facilitate a direct comparison to population risk and generalization error as defined in~\eqref{equ:gen}. Detailed derivations are provided, for reference, in Appendix~\ref{appx:AIC}.

\paragraph{BIC}
This criterion \cite{schwarz1978estimating} ranks statistical models by their
marginal likelihoods of generating the data,
with smaller values of the BIC correspond
to larger marginal likelihoods.
Approximating the marginal likelihood of observing $\vz^n$ for $M_k$, i.e., 
$m_k(\vz^n) \triangleq  \int P_k(y^n| \vx^n ;\vtheta_k) \, \pi_k(\vtheta_k) \, d \vtheta_k$, 
Laplace's method yields
\begin{equation*}
    \log m_k(\vz^n) =
    \hat{L}_{k}(\hat{\vtheta}_{\mathrm{ML}}^{(k)}) - \frac{p_k}{2} \log n + O(1),\quad n\to\infty.
\end{equation*}
In turn, BIC  
is obtained 
by dropping terms that do not scale with $n$ and scaling by $-1/n$:
\begin{align} \label{equ:BIC}
    \mathrm{BIC} &= -
    \frac{\hat{L}(\hat{\vtheta}_{\mathrm{ML}})}{n} + \frac{p\log n}{2n}.
\end{align}
When, further, $P(M_k)=1/K$, we obtain
\begin{equation}
    P(M_k|\vz^n)
    = \frac{m_k(\vz^n) P(M_k) }{\sum_{k=1}^K m_k(\vz^n) P(M_k)}
    \propto m_k(\vz^n).
\end{equation}
Thus, when we assume a uniform prior over different models, the BIC  ranks models by their posterior probability of generating the training data, and choosing the smallest BIC 
corresponds to the maximum a posteriori rule for model selection. A detailed derivation is provided in Appendix~\ref{appx:BIC}.



Both \eqref{equ:AIC} and \eqref{equ:BIC} share a common first term, representing the average negative log-likelihood of the training data for MLE, which can be interpreted as the empirical risk with log-loss, decreasing as we adopt more complex models. We note that AIC and BIC in the classical $n \to \infty$ regime are independent of the form of the model family $P(y|\vx;\vtheta)$ and the prior distribution $\pi(\vtheta)$, which makes it compatible with general distribution families subject to mild smoothness constraints. 
However, they select different models because AIC and BIC differ in the second penalty term.



\subsection{Information Risk Minimization and Gibbs Algorithm} \label{appx:gibbs}

Classical AIC and BIC depend on MLE, which can be viewed as an ERM solution that purely minimizes empirical risk. 
Instead, we motivate the Gibbs algorithm using an information risk minimization framework, which minimizes both empirical risk and a generalization error bound. 

We start with the following mutual information-based generalization error bound proposed in \cite{xu2017information}. 


%
  

\begin{lemma}[\cite{xu2017information}]\label{lemma:gen_bound}
Suppose the loss function $\ell(w,z) \in [0,1]$ is bounded, and $S=\{Z_i\}_{i=1}^n$ contains $n$ i.i.d. training samples, then $|\overline{gen}(P_{\hat{W}|S},P_S)|\leq \sqrt{I(\hat{W};S)/(2n)}$.
\end{lemma}
From Lemma $\ref{lemma:gen_bound}$, we know that the mutual information between training data $S$ and the learned parameter $\hat{W}$ can be used as an upper bound for generalization error. Thus, 
\cite{xu2017information} further considers the following algorithm that minimizes empirical risks regularized with mutual information,
\begin{equation}\label{equ:Infor_Mini_ML}
P_{\hat{W}|S}^{*}=\underset{P_{\hat{W}|S}}{\operatorname{argmin}}\  \mathbb{E}_{P_{\hat{W},S}}\big[L_{E}(\hat{W},S)\big] + \frac{1}{\beta}I(\hat{W};S),
\end{equation}
where $\beta >0$ controls the regularization term and balances between over-fitting and generalization. As $\beta \to \infty$, it reduces back to the standard ERM algorithm.

As computing $I(\hat{W};S)$ requires the knowledge of $P_{\hat{W}}$, \cite{zhang2006information,xu2017information,perlaza2022empirical} further relax \eqref{equ:Infor_Mini_ML} by replacing it with an upper bound $D(P_{\hat{W}|S}\|\pi|P_{S}) \ge I(\hat{W};S)$, where $
\pi$ is an arbitrary prior distribution over $\mathcal{W}$. The following lemma characterizes the solution to the relaxed problem.

\begin{lemma}[\cite{zhang2006information,perlaza2022empirical}]\label{lemma:Infor_mini} The minimum value of the following information risk minimization (IRM) problem is
\begin{align}\label{equ:Infor_Mini_KL}
\underset{P_{\hat{W}|S}}{\operatorname{min}}\  \mathbb{E}_{P_{\hat{W},S}}\big[L_{E}(\hat{W},S)\big] + \frac{1}{\beta}D(P_{\hat{W}|S}\| \pi|P_{S})\nonumber \\
= -\frac{1}{\beta}\mathbb{E}_{P_S}\big[ \log \mathbb{E}_{\pi}[ e^{-\beta L_{E}(W,S)}]\big],
\end{align}
which is achieved by the following Gibbs algorithm 
\begin{equation}\label{equ:Gibbs}
P_{\hat{W}|S}^{*}(w|s)=\frac{\pi(w) e^{-\beta L_{E}(w,s)}}{\mathbb{E}_{\pi}\big[ e^{-\beta L_{E}(W,s)}\big]},\ \text{ for }\ \beta >0.
\end{equation}
\end{lemma}
In addition, 
\cite{perlaza2023validation} shows that the minimum value of IRM can be decomposed into the empirical risk under the prior distribution and the relative entropy $D(\pi\|P_{\hat{W}|S})$. 
\begin{lemma}[\cite{perlaza2023validation}]\label{lemma:Infor_prior}
The Gibbs distribution $P_{\hat{W}|S}^{*}(w|s)$ in form $\eqref{equ:Gibbs}$ satisfies 
\begin{align}\label{equ:ERM_RER}
&\mathbb{E}_{\pi\otimes P_S }\big[L_{E}(W,S)\big] - \frac{1}{\beta}D( \pi\|P_{\hat{W}|S}^{*}|P_S)\nonumber,\\
= &-\frac{1}{\beta} \mathbb{E}_{P_S}\big[\log \mathbb{E}_{\pi}[ e^{-\beta L_{E}(W,S)}].
\end{align}
\end{lemma}

\subsection{Sampling Algorithms for Gibbs distribution}\label{sec:SGLD}

The Gibbs distribution was first proposed by \cite{gibbs1902elementary} in statistical mechanics and further explored in the context of information theory by \cite{jaynes1957information}. In general, it is difficult to compute Gibbs posterior directly due to the integral in the partition function $\mathbb{E}_{\pi}\big[ e^{-\beta L_{E}(W,s)}\big]$. 
In practice, we delve into two stochastic algorithms, stochastic gradient Langevin dynamics (SGLD) or Metropolis-adjusted Langevin algorithm (MALA), to obtain samples from the Gibbs distribution. 



\paragraph{SGLD} 
The SGLD algorithm is defined using the following update formula,
\begin{equation}\label{SGLD}
    \hat{W}_{t+1}=\hat{W}_{t}-\eta\nabla L_{E}(\hat{W}_{t},s)+\sqrt{\frac{2\eta}{\beta}}\zeta_{t},\ t=0,1,...,
\end{equation}
where $\zeta_{t}$ is a standard Gaussian random vector and $\eta\geq 0$ is the step size. SGLD can also be viewed as a noisy version of standard SGD.
It is shown in \cite{raginsky2017nonconvex} that under some conditions on loss function, the Wasserstein distance between the distribution $\smash{P_{\hat{W}_t|S}}$ induced by SGLD and the Gibbs distribution $\smash{P_{\hat{W}|S}^*}$ will converge to zero as $t$ goes infinity. This allows us to sample $\hat{W}$ from the Gibbs distribution using SGLD with sufficiently large training steps $t$. Note that SGLD has been applied similarly to optimize generalization bound in \cite{dziugaite2019entropysgd}.

\paragraph{MALA}
Another approach is the Metropolis-adjusted Langevin
algorithm (MALA)~\cite{dwivedi2018log}. MALA and SGLD are first-order sampling methods since they have similar gradient update formulas, which guarantees that both algorithms converge to the Gibbs distribution. 
MALA differs from the SGLD by introducing an additional Metropolis-adjusted step, which provides a faster
convergence rate, as shown in~\cite{dwivedi2018log,mangoubi2019nonconvex,holzmuller2023convergence}. A more detailed discussion of our implementations of SGLD and MALA can be found in Appendix~\ref{appx:SGLD}.


\section{Gibbs-based Information Criteria}
\label{cbic}
We now develop information-theoretic analyses for Gibbs-based AIC and BIC, following the same classical principles. AIC captures the population risk, and BIC approximates the log marginal likelihood. We demonstrate that our proposed Gibbs-based information criteria align with the classical information criteria in the traditional large $n$ regime and discuss how our information-theoretic analysis can be generalized to the over-parameterized regime.



\subsection{Gibbs-based AIC}

As we discussed in Section~\ref{sec:pre_AIC}, the penalty term in AIC can be viewed as the generalization error of MLE with log-loss. Thus, we start with the following result from \cite{aminian2021exact}, which provides an exact characterization for the generalization error of the Gibbs algorithm using information measure. 
\begin{proposition}
[\cite{aminian2021exact}]\label{prop:Gibbs_gen}
For the Gibbs algorithm defined in \eqref{equ:Gibbs},
the expected generalization error is 
\begin{equation}
    \overline{gen}(P_{\hat{W}|S}^{*},P_S) =  I_{\mathrm{SKL}}(P_{\hat{W}|S}^{*},P_S)/\beta,
\end{equation}
where $I_{\mathrm{SKL}}(P_{\hat{W}|S}^{*},P_S)$ is the symmetrized KL information between $\hat{W}$ and $S$, defined as follows
\begin{small}
\begin{equation*}
    I_{\mathrm{SKL}}(P_{Y|X},P_X) \triangleq D(P_{X,Y}\| P_{X}\otimes P_{Y}) + D( P_{X}\otimes P_{Y}\| P_{X,Y}).
\end{equation*}
\end{small}
\end{proposition}
Notably, information risk minimization in \eqref{equ:Infor_Mini_ML} regularizes the mutual information $I(\hat{W};S)$ as a proxy of the generalization error, but the exact generalization error of the Gibbs algorithm is the symmetrized KL information, which is always larger than the mutual information. 


As discussed in Section~\ref{sec:SGLD}, we can obtain samples from the Gibbs distribution with SGLD or MALA. Thus, the population risk of the Gibbs algorithm can be approximated by   
$L_{P}(\hat{W}_{\mathrm{Gibbs}},P_{Z}) \approx L_{E}(\hat{W}_{\mathrm{Gibbs}},\vz^n) + \frac{1}{\beta}I_{\mathrm{SKL}}(P_{\hat{W}|S}^{*},P_S)$,
which motivates the following \emph{Gibbs-based} AIC:
\begin{equation}\label{equ: Gibbs_AIC}
    \mathrm{AIC^{+}} \triangleq L_{E}(\hat{W}_{\mathrm{Gibbs}},\vz^n) + \frac{1}{\beta}I_{\mathrm{SKL}}(P_{\hat{W}|S}^{*},P_S).
\end{equation}
Observe that the penalty term in Gibbs-based AIC is an information measure that captures its generalization error. By investigating the asymptotic behavior of ${I_{\mathrm{SKL}}(P_{\hat{W}|S}^{*},P_S)}$, we have the following theorem characterizes the Gibbs-based AIC in the classical asymptotic regime.


\begin{theorem} (proved in Appendix~\ref{appx:Gibbs_AIC}) \label{thm:AIC_classical}
Consider the log-loss $\ell(w,\vz) = -\log P(y|\vx;w)$, and set $\beta = n$. Under proper regularization assumptions in Appendix~\ref{appx:Gibbs_AIC}, the Gibbs-based AIC has the following form in the regime where $p$ is fixed and $n\rightarrow \infty$:
\begin{equation}\label{equ:AIC+classical}
    \mathrm{AIC^{+}} =  L_{E}(\hat{W}_{\mathrm{Gibbs}},\vz^{n})+ \frac{p}{n}.
\end{equation}
\end{theorem}
Evidently, our information-theoretic analysis has the same AIC penalty term for the $\mathrm{AIC^{+}}$ in the classical regime, which suggests that the generalization error of the Gibbs algorithm (SGLD or MALA) has the same order of $p/n$ as that of the MLE (SGD) in this regime.



\subsection{Gibbs-based BICs}


The Gibbs-based BIC is constructed by computing the marginal likelihood $m(\vz^n)$ using the information risk minimization framework. As such, it differs from the standard approach in classical (MLE-based) BIC, as no Laplace  approximation is needed.

We now show that the minimum value achieved by the Gibbs algorithm with log-loss in the information risk minimization is the negative log-marginal likelihood.
\begin{proposition}\label{prop:Gibbs_marginal}(proved in Appendix~\ref{appx:Gibbs_BIC})
For the Gibbs algorithm $P^*_{\hat{W}|S}$ defined in \eqref{equ:Gibbs}, if we adopt the log-loss function $\ell(w,\vz) = -\log P(y|\vx;w)$, and set $\beta = n$,  the marginal likelihood is 
\begin{align}
    &-\frac{1}{n} \log m(\vz^n)\nonumber \\ 
    &=\mathbb{E}_{P^*_{\hat{W}|S=\vz^n}}\big[L_{E}(\hat{W},\vz^n)\big] + \frac{1}{n}D(P^*_{\hat{W}|S=\vz^n}\| \pi)\\
    &=\mathbb{E}_{\pi}\big[L_{E}(\hat{W},\vz^n)\big] - \frac{1}{n}D(\pi \| P^*_{\hat{W}|S=\vz^n}).\nonumber
\end{align}
\end{proposition}


Motivated by Proposition \ref{prop:Gibbs_marginal}, we propose two different Gibbs-based BICs to approximate the marginal likelihood:  
\begin{equation}  \label{equ: Gibbs_BIC}  
    \mathrm{BIC^{+}} \triangleq L_{E}(\hat{W}_{\mathrm{Gibbs}},\vz^n) + \frac{1}{n}D(P^*_{\hat{W}|S=\vz^n}\|\pi),
\end{equation}
\begin{equation}  \label{equ: Gibbs_BIC-}  
    \mathrm{BIC^{-}} \triangleq \mathbb{E}_{\pi}\big[L_{E}(W,\vz^n)\big] - \frac{1}{n}D(\pi\|P^*_{\hat{W}|S=\vz^n}).
\end{equation}
Interestingly,  given that the empirical risk term is evaluated under the predetermined prior distribution, it is unnecessary to sample from the Gibbs posterior to evaluate $\mathrm{BIC^{-}}$. The $\mathrm{BIC^{-}}$ can be directly obtained by computing the second KL divergence term depending on $ P^*_{\hat{W}|S=\vz^n}$.

To compare $\mathrm{BIC^{+}}$ with the classical BIC, it suffices to investigate the asymptotic behavior of the KL divergence between the Gibbs posterior distribution and the prior as $n\to \infty$. The final expression of the Gibbs-based $\mathrm{BIC^{+}}$ is formally stated in the following Theorem. 



\begin{theorem}(proved in Appendix~\ref{appx:Gibbs_BIC_classical}) \label{thm:BIC_Gibbs_classical}
Under proper regularization assumptions in Appendix~\ref{appx:Gibbs_BIC_classical}, the Gibbs-based $\mathrm{BIC^{+}}$ has the following form in the classical regime where $p$ is fixed and $n\rightarrow \infty$,
\begin{equation}\label{equ:BIC+classical}
    \mathrm{BIC^{+}} \triangleq  L_{E}(\hat{W}_{\mathrm{Gibbs}},\vz^{n})+ \frac{p}{2n}\log n.
\end{equation}
\end{theorem}

As expected, we have the same BIC penalty term for the $\mathrm{BIC^{+}}$ in the classical regime. The experimental results in Appendix~\ref{appx:exp} show that the proposed Gibbs-based AIC and BIC are comparable to their classic counterparts. In the over-parameterized regimes, they are not, as we now develop. 



\section{BICs for Over-Parameterized RF Model}\label{sec:rbic}

As shown in Figure \ref{fig:over_loss} (right), the classical AIC and BIC fail to capture the double-descent trend exhibited by the test MSE of the RF model. This is due to the fact that the generalization error and marginal likelihood have different behaviors in the over-parameterized regime, which cannot be characterized by the classical
asymptotic analysis. The classical analyses heavily rely on the asymptotic normality of MLE and Laplace approximation under certain regularization assumptions, which ignores the prior distribution as $n\to \infty$. Unfortunately, none of these properties hold in the over-parameterized regime where $p \gg n$, as there exist an infinite number of possible model parameters that could interpolate $n$ samples perfectly, making the classical AIC and BIC ill-defined. 

However, the Gibbs-based $\mathrm{AIC^{+}}$ in \eqref{equ: Gibbs_AIC}, $\mathrm{BIC^{+}}$ in \eqref{equ: Gibbs_BIC}, and $\mathrm{BIC^{-}}$ in \eqref{equ: Gibbs_BIC-} defined using different information measures can be generalized to over-parameterized models, as Proposition~\ref{prop:Gibbs_gen} and~\ref{prop:Gibbs_marginal} hold regardless of the values of $p$ and $n$. Since $\mathrm{AIC^{+}}$ mainly captures the generalization error of the Gibbs algorithm, which can be estimated using validation data in practice, our focus lies in extending the analysis of $\mathrm{BIC^{+}}$ and $\mathrm{BIC^{-}}$ to approximate the marginal likelihood in the over-parameterized regime. 



Owing to the complex nature of fitting in the over-parameterized regime, we do not pursue a general asymptotic formula that applies to all model architectures, as in the classic regime. Instead, we refine the Gibbs-based BIC analysis to this regime for the random feature (RF) model.



\subsection{Random Feature Model}
\label{sec:model}
The RF model \cite{rahimi2008random} takes the form of a two-layer neural network with fixed random weights in the first layer.
The output of RF model with input data $\vx \in \mathbb{R}^d$ is
\begin{equation}\label{equ:RF_model}
    {g}(\vx) \triangleq \sum_{j=1}^p f \Big(\frac{ \langle \vx, \mF_j \rangle}{\sqrt{d}} \Big) \vw_j
    = f\Big(\frac{\vx^\top \mF}{\sqrt{d}}\Big) \vw,
\end{equation}
where 
 $\vw \in \mathbb{R}^p$ denotes the weights of the model. Moreover,
$\mF_j \in \mathbb{R}^d$ denotes the $j$\/th random feature vector,
which is the $j$\/th column of the random feature matrix $\mF \in \mathbb{R}^{d \times p}$
whose entries are drawn i.i.d.\ from $\mathcal{N}(0, 1)$. Finally,
$f(\cdot)$ is a point-wise activation function. In our setting, there are $n$ training samples
$\vz^n = \{(\vx_i,y_i)\}_{i=1}^n$,
and the $\vx_i$ are drawn i.i.d.\ from $\mathcal{N}(0, \mI_d)$.

The parametric distribution family induced by the random feature model takes the form
\begin{align}
    &P(y^n|\vx^n;\vw) \\
    &=\frac{1}{(2\pi \sigma^2)^{n/2} } \exp \Big(-\frac{1}{2\sigma^2} \sum_{i=1}^n\big(y-f\big(\frac{  \vx_i^\top \mF}{\sqrt{d}}\big) \vw\big)^2\Big),\nonumber
\end{align}
with a fixed random feature matrix $\mF$,  and a Gaussian prior distribution $\vw \sim \mathcal{N}(0, \frac{\sigma^2}{\lambda n}\mI_p)$. 
The weights of the RF model $\vw$ 
can be trained using the
regularized log-loss
\begin{align}\label{equ:RF_loss}
        \mathcal{L}(\vw) &= \frac{1}{2\sigma^{2}}\|\mY-\mB \vw \|_2^2  +\frac{n}{2} \log (2\pi \sigma^{2}) +\frac{\lambda n \|\vw\|_2^2}{2\sigma^{2}},\nonumber \\ \quad 
        &\text{ where } \mB \triangleq f(\mX \mF/\sqrt{d}) \in\mathbb{R}^{n \times p},
\end{align}
and 
we collect the training data in a matrix
$\mX \in \mathbb{R}^{n \times d}$
and a vector $\mY \in \mathbb{R}^n$ to simplify the notation.

As discussed in \cite{mei2022generalization,d2020triple}, a significant benefit of 
using the random feature model is that, unlike in standard linear regression, the dimensionality of the input data $d$ is not entangled with the number of parameters $p$.



\begin{figure*}[!t]
    \centering
    \includegraphics[width=0.32\textwidth]{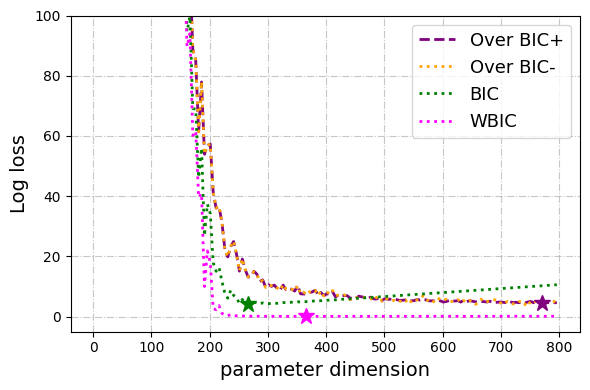}\hspace*{1em}
    \includegraphics[width=0.32\textwidth]{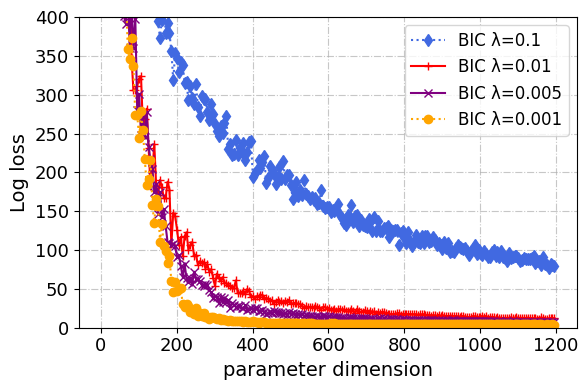}\hspace*{1em}
    \includegraphics[width=0.32\textwidth]{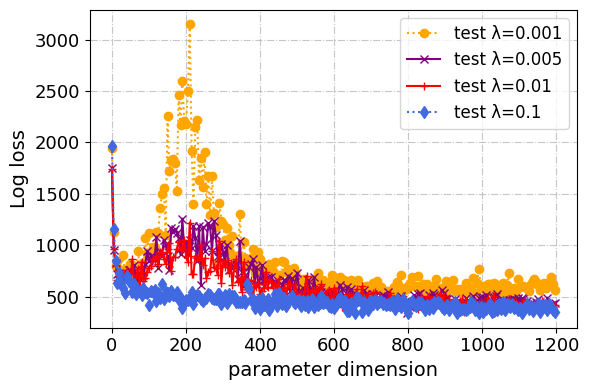}
    \vspace{-1em}
    \caption{A comparison between different $\mathrm{BIC}$s in over-parameterized RF model when $\lambda=0.001$ (\textbf{left}); A comparison between $\mathrm{BIC^{+}}$ (\textbf{middle}) and population risk (\textbf{right}) with varying $\lambda$.}
    \label{fig:BIC_comparasion}
    \vspace{-0.5em}
\end{figure*}

\subsection{Gibbs-based BICs for Over-Parameterized RF Model}
To generalize Gibbs-based BICs to the over-parameterized RF model, it suffices to focus on the second KL-divergence term in~\ref{equ: Gibbs_BIC} and~\ref{equ: Gibbs_BIC-}. In the random feature model, it can be shown (see Appendix~\ref{app:RF_posterior} for details) that the Gibbs algorithm reduces to the Gaussian posterior distribution $\smash{P_{\hat{W}|S}^{*} \sim \mathcal{N}(\hat{W}_{\lambda},\mSigma_w)}$, with the mean ${\hat{W}_{\lambda} = (\lambda n\mI_p+\mB^\top \mB)^{-1} \mB^\top \mY}$, and covariance matrix $\smash{\mSigma_w =\sigma^{2}(\lambda n\mI_p+\mB^\top \mB)^{-1}}$.

Thus, the KL-divergence between the Gibbs posterior distribution and prior $\mathcal{N}(0, \frac{\sigma^2}{\lambda n}\mI_p)$ is given by
\begin{equation}\label{equ:kl_bic}
   D(P^*_{\hat{W}|S=\vz^n}\|\pi)= 
   \frac{1}{2}\Big[\frac{\lambda n}{\sigma^{2}}\|\hat{W}_{\lambda} \|_2^2 +\log \frac{\det(\frac{\sigma^{2}}{\lambda n} \mI_p)}{\det(\mSigma_w)} + \mathrm{tr}(\frac{\lambda n}{\sigma^{2}} \mSigma_w)  - p \Big].  \nonumber
\end{equation}
And the the KL-divergence between the prior $\mathcal{N}(0, \frac{\sigma^2}{\lambda n}\mI_p)$ and Gibbs distribution is
given by
\begin{equation}\label{equ:kl_bic-}
   D(\pi \| P^*_{\hat{W}|S=\vz^n}) = 
   \frac{1}{2}\Big[\hat{W}_{\lambda}^{\top}(\mSigma_w)^{-1}\hat{W}_{\lambda}  -\log \frac{\det(\frac{\sigma^{2}}{\lambda n} \mI_p)}{\det(\mSigma_w)}+ \mathrm{tr}(\frac{\sigma^{2}}{\lambda n}\mSigma_w^{-1}) - p \Big]. \nonumber
\end{equation}
To obtain a convenient expression for the remaining determinant and trace terms, we first impose restrictions on the activation function $f(\cdot)$. Therefore, these two terms can be characterized using random matrix theory by studying the eigenvalues of $\mSigma \triangleq \mB^\top \mB/(\lambda n) + \mI_p$ in the over-parameterized regime.
In particular, for activation functions $f(\cdot)$
that satisfy conditions
\begin{align}
    \label{equ:act-func-conditions}
\mathbb{E}[f(\varepsilon)]=0,\quad  \mathbb{E}[f(\varepsilon)^2]&=1, \quad
        \mathbb{E}[f'(\varepsilon)]=0, \quad \nonumber \\
\big|\mathbb{E}[f(\varepsilon)^k]\big|&<\infty,\quad \text{ for $k>1$},
\end{align}
where  $\varepsilon \sim \mathcal{N}(0,1)$,
the following theorem characterizes the KL divergence term in the over-parameterized RF model. 



\begin{theorem}\label{thm:over_BIC}
For activation functions $f(\cdot)$ satisfying
the conditions in \eqref{equ:act-func-conditions},
as $n,d,p \to \infty$ with $p/d \to r_1$, $n/d \to r_2$, and $r_1/r_2 =r$, where $r_1,r_2 \in (0,\infty)$, we have
\begin{align}
&\frac{1}{n}D(P^*_{\hat{W}|S=\vz^n}\|\pi)  \nonumber\\
&\to\frac{\lambda}{2\sigma^2} \|\hat{W}_{\lambda}\|_2^2  -\frac{ \lambda }{8} \mathcal{F}(\frac{1}{\lambda},r)+ \frac{1}{2}{V}(1/\lambda, r)    
\end{align}
almost surely, where
\begin{align}
  V(\gamma, r) \triangleq &\ r \log \Big(1+ \gamma -\frac{1}{4} \mathcal{F}(\gamma,r) \Big) -\frac{1}{4\gamma}\mathcal{F}(\gamma,r)\nonumber\\
  &+ \log \Big(1+ \gamma r -\frac{1}{4} \mathcal{F}(\gamma,r) \Big),
\end{align}
\begin{align*}
    \mathcal{F}(\gamma,r) \triangleq
    \Big(\sqrt{\gamma(1+\sqrt{r})^2+1} - \sqrt{\gamma(1-\sqrt{r})^2+1}\Big)^2.
\end{align*}
\end{theorem}

\begin{proof}[Sketch of Proof]
The proof is based on the results from \cite{pennington2017nonlinear},
which shows that the distribution of the eigenvalues of the random matrix $\mB^\top\mB/n$
converges to the Marchenko-Pastur distribution with shape parameter $r$
(a well-studied distribution in random matrix theory~\cite{tulino2004random}). 
The detailed proof is provided in Appendix~\ref{app:RF}.
\end{proof}

\begin{remark}
An example of an activation function that
satisfies all the assumptions we made in  \eqref{equ:act-func-conditions}
is $f(x) = (x^2-1)/\sqrt{2}$.
More examples of such activation functions can be found in \cite{pennington2017nonlinear}. We further note that the assumption  $\mathbb{E}[f'(\varepsilon)]=0$ on the activation function is used only to obtain a simple closed-form for the KL divergence, a more general result by considering the Stieltjes transform of $\mB^\top\mB/n$ for other activation functions is provided in Appendix~\ref{app:RF}.
\end{remark}

Theorem~\ref{thm:over_BIC} motivates us to define the following Gibbs-based BICs for the over-parameterized RF model to approximate the marginal likelihood, 
\begin{align}\label{equ:BICover}
\mathrm{BIC^{+}}\triangleq\ & L_{E}(\hat{W}_{\mathrm{Gibbs}},\vz^{n}) +\underbrace{\frac{\lambda}{2\sigma^2} \|\hat{W}_{\mathrm{\lambda}}\|_2^2}_\text{$\ell_2$ term}  \\
&\underbrace{-\frac{ \lambda }{8} \mathcal{F}(\frac{1}{\lambda},r)+ \frac{1}{2}{V}(1/\lambda, r)}_\text{covariance term}.  \nonumber \\ 
\mathrm{BIC^{-}}\triangleq\ \label{equ:BIC-over}
&  \mathbb{E}_{\pi}\big[L_{E}(W,\vz^n)\big] -  \frac{1}{2n}\hat{W}_{\lambda}^{\top}(\mSigma_w)^{-1}\hat{W}_{\lambda}  \\
&- \frac{1}{2n}\mathrm{tr}( \mI_p+\frac{1}{\lambda n}\mB^\top \mB\big)+ \frac{1}{2}{V}(1/\lambda, r)+\frac{p}{2n}.    \nonumber
\end{align}

Depending on the specific sampling method employed, as detailed in Section~\ref{sec:SGLD}, the term $L_{E}(\hat{W}_{\mathrm{Gibbs}},\vz^{n})$ can be estimated with either $L_{E}(\hat{W}_{\mathrm{SGLD}},\vz^{n})$ or $L_{E}(\hat{W}_{\mathrm{MALA}},\vz^{n})$. When $\sigma$ is comparatively small and  $\lambda$ is comparatively large,  $\|\hat{W}_{\mathrm{\lambda}}\|_2^2$ can be substituted by $ \|\hat{W}_{\mathrm{SGLD}}\|_2^2$ or $ \|\hat{W}_{\mathrm{MALA}}\|_2^2$, as the posterior is primarily driven by the mean.
The penalty term of $\mathrm{BIC^{+}}$ consists of the $\ell_2$ norm of the learned weights, and two other terms capture the log determinant and trace in the over-parameterized regime, which will be referred to as the covariance term altogether in the next section. 

\section{Experiment and Discussion}\label{sec:exp_over}
\begin{figure*}[!t]
    \centering
    \includegraphics[width=0.4\textwidth]{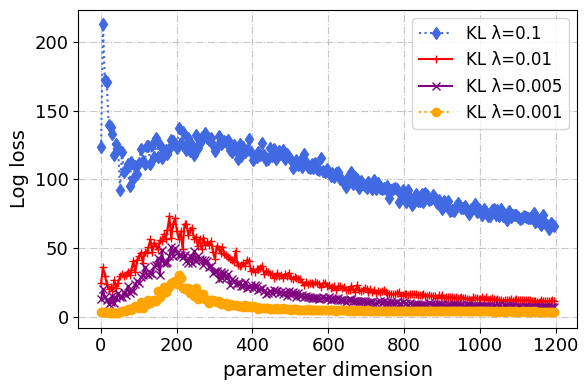
    }\hspace*{2em}
    \includegraphics[width=0.4\textwidth]{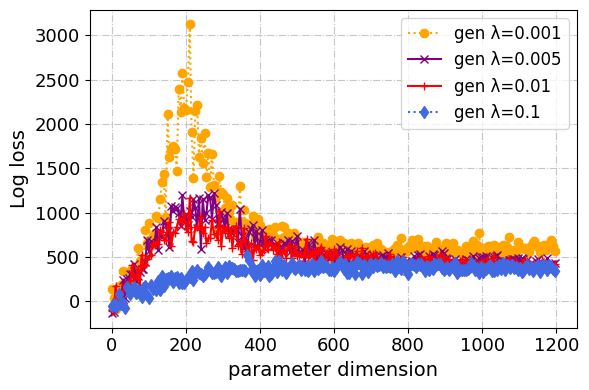}
    \vspace{-0.5em}
    \caption{A comparison between the KL-divergence  term in $\mathrm{BIC}^+$ (\textbf{left}) and the generalization error term in $\mathrm{AIC}^+$ (\textbf{right}) with varying $\lambda$.}
    \label{fig:KL_ISKL}
\end{figure*}

\begin{figure*}[!t]
    \centering
    \includegraphics[width=0.4\textwidth]{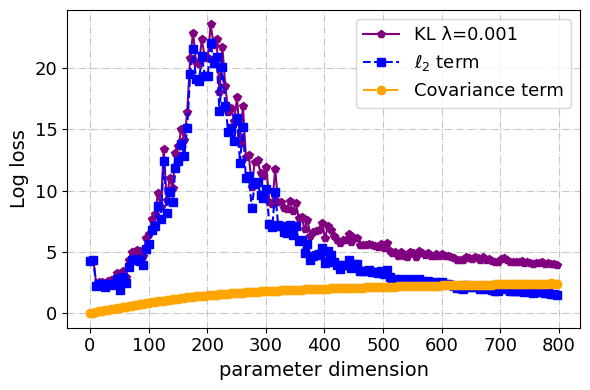}\hspace*{2em}
    \includegraphics[width=0.4\textwidth]{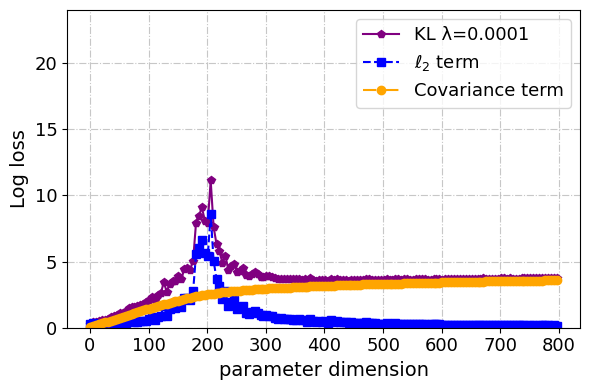}
    \vspace{-0.5em}
    \caption{A decomposition of the terms in over-parameterized $\mathrm{BIC}^+$ in~\eqref{equ:BICover} with $\lambda=0.001$ (\textbf{left}), and  $\lambda=0.0001$(\textbf{right}).}
    \label{fig:BIC_decomp}
    \vspace{-0.5em}
\end{figure*}


In this section, we instantiate a two-layer RF model described in~\eqref{equ:RF_model}, where the first layer is designated for feature mapping and is kept random, and we only train the parameter in the second layer. See Appendix~\ref{appx:exp} for experimental details and additional results.

We evaluate the over-parameterized Gibbs-based $\mathrm{BIC}^+$ and $\mathrm{BIC}^-$ using $n=200$ samples generated by the linear model
\begin{equation*}
     y_i = \vx_i^\top \vw^* +\epsilon_i,\quad \vw^*\in \mathbb{R}^d,\ \|\vw^*\|_2^2=1,\ \epsilon_i \sim \mathcal{N}(0, \sigma^2),
\end{equation*}
with input dimension $d=400$ and noise $\sigma^2=0.1$, and we use MALA to sample from Gibbs distribution. 

\textbf{Selection of $\sigma^{2}$}
By examining the mean of the Gibbs posterior, given as ${\hat{W}_{\lambda} = (\lambda n\mI_p+\mB^\top \mB)^{-1} \mB^\top \mY}$, and its covariance $\smash{\mSigma_w =\sigma^{2}(\lambda n\mI_p+\mB^\top \mB)^{-1}}$, it becomes evident that $\sigma^2$ impacts only the variance of the posterior and not its mean. A large $\sigma^2$ can introduce instability in the sampling from the posterior and deteriorate the test performance. 
Thus, in our configuration, we opt for a smaller $\sigma=0.05$ to ensure good model performance, even if it results in a larger scaling factor for empirical risk. In our upcoming discussion, We will investigate how the other prior parameter $\lambda$ influences the posterior.

\textbf{Double-descent of population risk}
As depicted in Figure~\ref{fig:over_loss} (left), the peak of test loss is located at the interpolation threshold, i.e., when $p=n=200$, resulting in the highest generalization error. As $p$ continues to increase, the test error begins to decline again, even falling below the levels observed in the under-parameterized regime $p<n$.

\textbf{Comparison of different BICs}
In the right panel of Figure~\ref{fig:over_loss} and Figure~\ref{fig:BIC_comparasion} (left), it is evident that the classic $\mathrm{BIC}$ prefers moderate model due to the ill-defined issue we discussed before. In Figure~\ref{fig:BIC_comparasion} (left), we compare our Gibbs-based BICs with the WBIC proposed in~\cite{watanabe2013widely}. The analysis of
WBIC is based on approximating the marginal likelihood for singular models in the classical large $n$ regime, and such an approximation becomes sensitive when the prior parameter $\lambda$ is small. Thus, WBIC also fails to capture the marginal likelihood, whereas our over-parameterized $\mathrm{BIC}^+$ and $\mathrm{BIC}^-$  in~\eqref{equ:BICover} and~\eqref{equ:BIC-over} based on the exact forms of the marginal likelihood, succeeds in selecting the large model.

\textbf{Mismatch between BIC and population risk} Note that even though the Gibbs-based BICs might favor a larger model for increased values of $p$, it does not exhibit double-descent behavior. This discrepancy becomes evident when examining Figure~\ref{fig:BIC_comparasion} (middle) and Figure~\ref{fig:BIC_comparasion} (right), where BIC exhibits a distinctive pattern compared to the population risk. A similar mismatch between marginal likelihood and generalization for ERM has been observed in~\cite{lotfi2022bayesian}.

We further investigate the inconsistency between the marginal likelihood and population risk for the Gibbs algorithm in Figure~\ref{fig:KL_ISKL}. Unlike the classical BIC, where the penalty term $p\log n/(2n)$ is order-wise larger than the $p/n$ term in classical AIC, it can be seen that the KL divergence term in the over-parameterized $\smash{\mathrm{BIC}^+}$ can be significantly smaller than the generalization error $I_{\mathrm{SKL}}$ in Figure~\ref{fig:KL_ISKL}, depending on the value of $\lambda$ and $p$. Thus, the mismatch between marginal likelihood (BIC) and population risk (AIC) is even more complicated in the over-parameterized setting due to the influence of prior distribution.

\textbf{Similarity between KL divergence and generalization}
When comparing the KL divergence term of $\mathrm{BIC}^+$ in Figure~\ref{fig:KL_ISKL} (left) and the generalization error in Figure~\ref{fig:KL_ISKL} (right) with the same $\lambda$, a similar trend emerges as $p$ increases. To understand such similarity between the two terms, we decompose the penalty term of the over-parameterized $\mathrm{BIC}^+$ in Figure~\ref{fig:BIC_decomp} into $\ell_2$ term, covariance term.
 
As shown in~Figure~\ref{fig:BIC_decomp}, when $p\leq n$, the model prior can be ignored, and the training loss becomes the dominant factor of $\mathrm{BIC}^+$ (training loss is plotted in Appendix~\ref{appx:exp} due its large scale). In this case, the KL divergence and the covariance term increase with $p$, corresponding to the classical BIC.  
When $p\geq n$, the KL divergence penalty term dominates the over-parameterized $\mathrm{BIC}^+$. In this regime, multiple weights can fit the training data perfectly. From Figure~\ref{fig:BIC_decomp}, we note that regardless of the chosen $\lambda$, the $\ell_2$ term exhibits double-descent behavior and decreases as p increases. This suggests that the Gibbs algorithm prefers the weights with smaller $\ell_2$ norm to fit the data. Note that similar phenomena are observed for SGD, and generalization error bounds using different weights norms are established in \cite{neyshabur2017exploring,bartlett2017spectrally}. Thus, a smaller $\ell_2$ norm can lead to a better generalization performance. Thus, the behavior of the $\ell_2$ norm may elucidate why, for a given $\lambda$ value, both KL divergence and the generalization error $I_{\mathrm{SKL}}$ exhibit similar double-descent behavior.

\textbf{Mismatch between KL divergence and generalization}
The comparison between Figure~\ref{fig:KL_ISKL} (left) and Figure~\ref{fig:KL_ISKL} (right) with different values of $\lambda$ reveals a mismatch between KL divergence and generalization error. A closer examination of Figure~\ref{fig:BIC_decomp} (left) and Figure~\ref{fig:BIC_decomp} (right) shows that smaller $\lambda$ results in a smaller $\ell_2$ term in KL divergence. Note that $\|\hat{W}_{\mathrm{\lambda}}\|_2^2$ is increasing with smaller $\lambda$, but the $\ell_2$ term in KL divergence equals to $\frac{\lambda}{2\sigma^2} \|\hat{W}_{\mathrm{\lambda}}\|_2^2$. Consequently, the KL divergence places a more substantial penalty on larger values of $\lambda$, which explains the mismatch between the KL divergence and $I_{\mathrm{SKL}}$ with different priors by varying $\lambda$.  A comprehensive review can be found in the experimental results section, Appendix~\ref{appx:exp}.





\newpage
\bibliographystyle{ieeetr}
\bibliography{ref}

\newpage
\appendix

\section{Derivation of classical AIC}\label{appx:AIC}
We start by formally presenting the regularization conditions required for the standard asymptotic normality of MLE in the classical regime.   
\begin{assumption}\label{assump:MLE}
\textbf{Regularity Conditions for MLE}
\begin{enumerate}
\item Identifiability: $P(y| \vx;\vtheta) \neq P(y| \vx;\vtheta^{'})$ for $\vtheta \neq \vtheta^{'}$. 
\item $\Theta$ is an open subset of $\mathbb{R}^{p}$.
\item The function $\log P(y| \vx ;\vtheta)$ is three times continuously differentiable with respect to $\vtheta$.
\item There exist functions $F_{1}(z):\mathcal{Z} \rightarrow \mathbb{R}$,$F_{2}(z):\mathcal{Z} \rightarrow \mathbb{R}$ and $M(z):\mathcal{Z} \rightarrow \mathbb{R}$, such that
\begin{equation*}
    \mathbb{E}_{Z \sim P(\vz;\vtheta)} \big[M(Z)\big] < \infty,
\end{equation*}
and the following inequalities hold for any $\vtheta \in \Theta$,
\begin{align*}
 \left|\frac{\partial \log P(y| \vx ;\vtheta)}{\partial \theta_i} \right|  &< F_{1}(z), \quad  
  \left|\frac{\partial^{2} \log P(y| \vx ;\vtheta)}{\partial \theta_i \partial \theta_j} \right|  < F_{1}(z), \\
  \left|\frac{\partial^{3} \log P(y| \vx ;\vtheta)}{\partial \theta_i \partial \theta_j \partial \theta_k} \right|  &< M(z),\quad i,j,k=1,2,\cdots,p.
\end{align*}

\item The following inequality holds for an arbitrary $\vtheta \in \Theta$ and $i,j=1,2,...,p$,
\begin{equation*}
    0 < \mathbb{E} \Big[ \frac{\partial \log P(y| \vx ;\vtheta)}{\partial \theta_i} \frac{\partial \log P(y| \vx ;\vtheta)}{\partial \theta_j} \Big] < \infty.
\end{equation*}
\end{enumerate}
\end{assumption}

In the following, we provide proof of the classical AIC for reference. 

The AIC model selection in \eqref{equ:approx_kl_divergence} is equivalent to:
\begin{align}\label{eq:Expan_AIC}
\underset{k}{\operatorname{argmin}}\  \mathbb{E}_{P_{Z}}\big[-\log P_k(y| \vx ;\hat{\vtheta}_{\mathrm{ML}}^{(k)})\big]
&=\underset{k}{\operatorname{argmin}}\ \mathbb{E}_{P_{S}}\big[L_{E}(\hat{\vtheta}_{\mathrm{ML}}^{(k)},S)\big]+ \overline{gen}(\hat{\vtheta}_{\mathrm{ML}}^{(k)}, P_{Z}).
\end{align}
As $n\rightarrow \infty$, under the above regularization conditions, which guarantee that $\hat{\vtheta}_{\mathrm{ML}}$ is unique, the asymptotic normality of the MLE states that the distribution of $\hat{\vtheta}_{\mathrm{ML}}$ converges to 
\begin{equation}\label{equ:MLE_normal}
    \mathcal{N}(\vtheta^{*},\frac{1}{n}J(\vtheta^{*})^{-1}I(\vtheta^{*})J(\vtheta^{*})^{-1}), \text{ with   } \vtheta^{*} \triangleq \argmin\limits_{\vtheta \in \Theta} D(P_Z\|P(y|x;\vtheta)),
\end{equation}
where
\begin{equation}
    J(\vtheta) \triangleq \mathbb{E}_{P_{Z}}\big[-\nabla^2_{\vtheta}\log P(y| \vx ;\vtheta)\big]\text{ and }I(\vtheta)\triangleq \mathbb{E}_{P_{Z}}\big[\nabla_{\vtheta}\log P(y| \vx ;\vtheta) \nabla_{\vtheta}\log P(y| \vx ;\vtheta)^\top\big].
\end{equation}
When the true model is in the parametric family $P_{Z}=P(y| \vx ;\vtheta^{*})$, we have $J(\vtheta^{*}) =I(\vtheta^{*})$, which is the Fisher information matrix.

Thus,  the generalization term can be written as
\begin{align}\label{eq:gen_AIC}
-\overline{gen}(\hat{\vtheta}_{\mathrm{ML}}, P_{Z})
&=\mathbb{E}_{P_{S}}\big[L_{E}(\hat{\vtheta}_{\mathrm{ML}},S)\big]-L_{P}(\hat{\vtheta}_{\mathrm{ML}},P_Z)\nn\\
&=\mathbb{E}_{P_{S}}\big[L_{E}(\hat{\vtheta}_{\mathrm{ML}},S) \big] -\mathbb{E}_{P_{S}}\big[L_{E}(\vtheta^{*},S)\big]+ \mathbb{E}_{P_{S}}\big[L_{E}(\vtheta^{*},S) \big] -L_{P}(\hat{\vtheta}_{\mathrm{ML}},P_Z)\nn\\
&= \mathbb{E}_{P_{S}}\big[L_{E}(\hat{\vtheta}_{\mathrm{ML}},S)) - L_{E}(\vtheta^{*},S)  \big]+L_{P}(\vtheta^{*},P_Z)   -L_{P}(\hat{\vtheta}_{\mathrm{ML}},P_Z).
\end{align}
As $\hat{\vtheta}_{\mathrm{ML}}$ minimizes $L_{E}(\hat{\vtheta}_{\mathrm{ML}},S)$, we take the Taylor expansion of $L_{E}(\vtheta^{*},S)$ around the point $\hat{\vtheta}_{\mathrm{ML}}$
\begin{equation}\label{eq:talor_empir}
    L_{E}(\vtheta^{*},S) =L_{E}(\hat{\vtheta}_{\mathrm{ML}},S) +\frac{1}{2} (\vtheta^{*}-\hat{\vtheta}_{\mathrm{ML}})^\top J(\hat{\vtheta}_{\mathrm{ML}}) (\vtheta^{*}-\hat{\vtheta}_{\mathrm{ML}})+\cdots.
\end{equation}
 \newpage
And the Taylor expansion of $L_{P}(\hat{\vtheta}_{\mathrm{ML}},P_Z)$ around $\vtheta^{*}$ yields
\begin{equation}\label{eq:talor_popl}
   L_{P}(\hat{\vtheta}_{\mathrm{ML}},P_Z)= L_{P}(\vtheta^{*},P_Z)+\frac{1}{2}(\vtheta^{*}-\hat{\vtheta}_{\mathrm{ML}})^\top J(\vtheta^{*})(\vtheta^{*}-\hat{\vtheta}_{\mathrm{ML}})+\cdots.
\end{equation}

If we use the quadratic approximation of \eqref{eq:talor_empir} and \eqref{eq:talor_popl} in \eqref{eq:gen_AIC}, we can get the following asymptotic expression for the generalization error
\begin{equation}
\overline{gen}(\hat{\vtheta}_{\mathrm{ML}}, P_{Z}) = \frac{1}{n}  \mathrm{tr}( I(\vtheta^{*})J(\vtheta^{*})^{-1}),
\end{equation}
where the last step is due to the asymptotic normality of the MLE, and  $J(\hat{\vtheta}_{\mathrm{ML}})$ will converge to its expectation by the consistency of MLE and the strong law of large numbers. For the case where the true model is in the parametric family $P_{Z}=P(y| \vx ;\vtheta^{*})$, $J(\vtheta^{*}) =I(\vtheta^{*})$, $\overline{gen}(\hat{\vtheta}_{\mathrm{ML}}, P_{Z}) =p/n$. Thus, plug the asymptotic generalization error term back in \eqref{eq:Expan_AIC} can get,
\begin{equation}
    \mathrm{AIC} =-
    \frac{\hat{L}(\hat{\vtheta}_{\mathrm{ML}})}{n} + \frac{p}{n}.%
\end{equation}
\section{Derivation of classical     
BIC}\label{appx:BIC}
The Taylor expansion of the log-likelihood function $\hat{L}(\vtheta)$ around $\hat{\vtheta}_{\mathrm{ML}}$ yields
\begin{equation}
    \hat{L}(\vtheta)=\hat{L}(\hat{\vtheta}_{\mathrm{ML}}) - \frac{n}{2} (\vtheta-\hat{\vtheta}_{\mathrm{ML}})^\top J(\hat{\vtheta}_{\mathrm{ML}}) (\vtheta-\hat{\vtheta}_{\mathrm{ML}})+\cdots.
\end{equation}
Similarly, we can expand the prior distribution $\pi(\vtheta)$ with  Taylor series around $\hat{\vtheta}_{\mathrm{ML}}$ as
\begin{equation}
    \pi(\vtheta)=\pi(\hat{\vtheta}_{\mathrm{ML}}) +(\vtheta-\hat{\vtheta}_{\mathrm{ML}})^\top \nabla \pi(\vtheta)\big|_{\vtheta={\hat{\vtheta}}_{\mathrm{ML}}} +\cdots.
\end{equation}

The Laplace approximation takes advantage of the fact
that when $n$ is sufficiently large,
the integrand is concentrated in a neighborhood of the mode of
$\hat{L}(\vtheta)$,
i.e., the maximum likelihood (ML) estimator
$\hat{\vtheta}_{\mathrm{ML}}$.
Thus,
\begin{align*}
    &m(\vz^n) \\
    &= \int \exp\{\hat{L}(\vtheta)\}\pi (\vtheta) d\vtheta \\
    & \approx \int \exp\Big\{\hat{L}(\hat{\vtheta}_{\mathrm{ML}}) - \frac{n}{2} (\vtheta-\hat{\vtheta}_{\mathrm{ML}})^\top J(\hat{\vtheta}_{\mathrm{ML}}) (\vtheta-\hat{\vtheta}_{\mathrm{ML}})\Big\}\Big(\pi(\hat{\vtheta}_{\mathrm{ML}}) +(\vtheta-\hat{\vtheta}_{\mathrm{ML}})^\top \nabla \pi(\vtheta)\big|_{\vtheta={\hat{\vtheta}}_{\mathrm{ML}}}\Big) d\vtheta \\
    & \overset{(a)}{\approx} \exp\{\hat{L}(\hat{\vtheta}_{\mathrm{ML}}) \}\pi (\hat{\vtheta}_{\mathrm{ML}})\int  \exp\Big\{- \frac{n}{2} (\vtheta-\hat{\vtheta}_{\mathrm{ML}})^\top J(\hat{\vtheta}_{\mathrm{ML}}) (\vtheta-\hat{\vtheta}_{\mathrm{ML}})\Big\} d\vtheta \\
    & = \exp\{\hat{L}(\hat{\vtheta}_{\mathrm{ML}}) \}\pi (\hat{\vtheta}_{\mathrm{ML}})(2\pi)^{p/2} n^{-p/2} |J(\hat{\vtheta}_{\mathrm{ML}})|^{-1/2},
\end{align*}
where $(a)$ follows from the following fact
\begin{equation}
    \int  (\vtheta-\hat{\vtheta}_{\mathrm{ML}})^\top \exp\Big\{- \frac{n}{2} (\vtheta-\hat{\vtheta}_{\mathrm{ML}})^\top J(\hat{\vtheta}_{\mathrm{ML}}) (\vtheta-\hat{\vtheta}_{\mathrm{ML}})\Big\} d\vtheta = 0.
\end{equation}
Taking the logarithm of this expression and multiplying it by $-\frac{1}{n}$, we obtain
\begin{equation}\label{equ:BIC_full}
    -\frac{1}{n}\log (m(\vz^n)) \approx -\frac{1}{n}\hat{L}(\hat{\vtheta}_{\mathrm{ML}}) + \frac{p}{2n} \log \frac{n}{2 \pi} + \frac{1}{n}\log |J(\hat{\vtheta}_{\mathrm{ML}})| -\frac{1}{n} \log \pi (\hat{\vtheta}_{\mathrm{ML}}).
\end{equation}

Note that $J(\hat{\vtheta}_{\mathrm{ML}})$ is a random Hessian matrix.
As $n \to \infty$, $J(\hat{\vtheta}_{\mathrm{ML}})$ will converge to its expectation by the consistency of MLE and the strong law of large numbers. Thus, both $\log |J(\hat{\vtheta}_{\mathrm{ML}})|$ and $\log \pi (\hat{\vtheta}_{\mathrm{ML}})$ have order less than $O(1)$
with respect to the sample size $n$,
and can be ignored in the BIC, which gives the following classical form of BIC:
\begin{equation}
    \mathrm{BIC} = -
    \frac{\hat{L}(\hat{\vtheta}_{\mathrm{ML}})}{n} + \frac{p\log n}{2n}.
\end{equation}

\section{Details of SGLD and Gibbs Algrotihm}\label{appx:SGLD}

In~\cite{raginsky2017nonconvex}, it is shown that the following SGLD updates 
\begin{equation}
    \hat{W}_{t+1}=\hat{W}_{t}-\eta\nabla F_{Z}(W)+\sqrt{\frac{2\eta}{\beta}}\zeta_{t},\ t=0,1,...,
\end{equation}
will converge to the Gibbs algorithm satisfying ${P_{\hat{W}|S}\propto \exp(-\beta F_{Z}(\hat{W}))}$.

A detailed description of Metropolis adjusted Langevin algorithm is provided in Algorithm~\ref{alg:MALA}.

\begin{algorithm}[ht]\label{alg:MALA}
\caption{Metropolis adjusted Langevin algorithm (MALA)}
\label{alg:Framwork}
\begin{algorithmic}[1]
\REQUIRE  Step size $\eta$ and a sample $w_0$ from a starting distribution $\mu_0$
\ENSURE Sequence $w_1$, $w_2$, ...
 \FOR {$i=0,1,...$ }
    \STATE $w_{i+1} \sim N \left( w_i - \eta\nabla f(w_i), \frac{2\eta}{\beta} I_d \right)$  \% Propose a new state
    \STATE 
    $\alpha_{i+1} = \min \left\{ 
        1, 
        \frac{
            \exp \left( -f(w_{i+1}) - \beta \left\| w_i - w_{i+1} + \eta \nabla f(w_{i+1}) \right\|^2_2/4\eta \right)
        }{
            \exp \left( -f(w_i) - \beta \left\| w_{i+1} - w_i +\eta \nabla f(w_i) \right\|^2_2/4\eta \right)
        }
    \right\}$
    \STATE $U_{i+1} \sim U[0, 1]$\\
    \IF{$U_{i+1} \leq \alpha_{i+1}$}
        \STATE $w_{i+1} \leftarrow w_{i+1}$ \% accept the sample
        \STATE $w_{i+1} \leftarrow w_i$ \% reject the sample
    \ENDIF 
\ENDFOR
\end{algorithmic}
\end{algorithm}

Note that our Gibbs algorithm contains an arbitrary prior distribution $\pi$, i.e., $P_{\hat{W}|S}\propto\pi(\hat{W}) \exp^{-nL_{E}(\hat{W},S)}$. To have the same format as $F_{Z}(W)$, we let
\begin{equation}\label{L_E_hat}
    F_{Z}(w)=L_{E}(w,s) - \frac{1}{n} \log\pi(w).
\end{equation}
Thus, the loss function used in the SGLD or MALA update becomes 
\begin{equation}\label{loss}
    \ell(w,z_{i})=-\log P(y_{i}|x_{i},w) -\frac{1}{n}\log\pi(w).
\end{equation}
When the prior follows a Gaussian distribution $\pi(w)\sim \mathcal{N}(0, \frac{\sigma^{2}}{\lambda n} I_{p})$, the second term in the loss function can be viewed as a regularization term derived from the log prior. It is crucial to notice that the empirical log loss in the subsequent SGLD or MALA computation only includes the $L_{E}(w,s)$ term and does not incorporate the regularization term. 
As shown in \cite{raginsky2017nonconvex}, the regularizer term does not violate the assumptions used in the loss function, ensuring that our loss function will also converge to the desired Gibbs distribution.

\begin{figure*}[!t]
    \centering
    \includegraphics[width=0.4\textwidth]{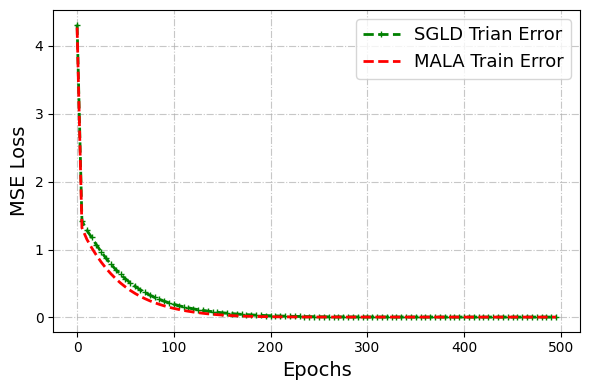}
    \hspace*{2em}
    \includegraphics[width=0.4\textwidth]{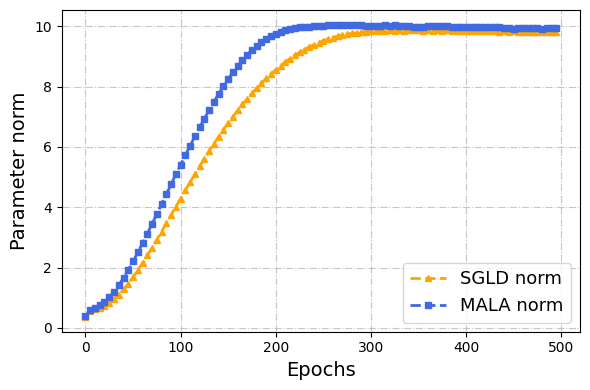}
    \vspace{-0.5em}
    \caption{The training MSE losses of MALA and SGLD are presented on the left (\textbf{left}), with a comparative analysis of their L2 norms shown on the right (\textbf{right}). These results were obtained for $\mathbf{n=200}$, using a learning rate of 0.001 and a parameter dimension of $p=600$ with $\lambda=0.01$, across ten runs. In terms of computational time, SGLD required an average of 10.4s to converge over 300 epochs, while MALA took an average of 7s to converge over 200 epochs.
     }
    \label{fig:MALA_SGLD_200}
\end{figure*}

\begin{figure*}[!t]
    \centering
    \includegraphics[width=0.4\textwidth]{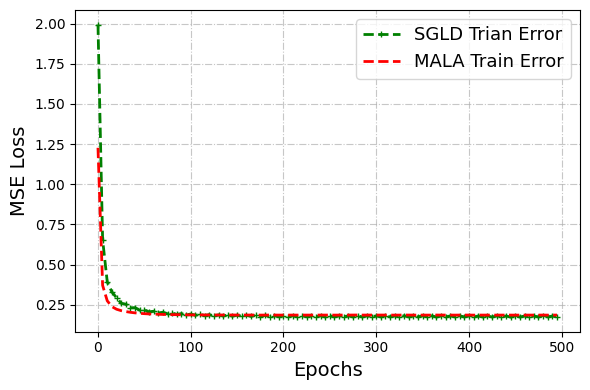}
    \hspace*{2em}
    \includegraphics[width=0.4\textwidth]{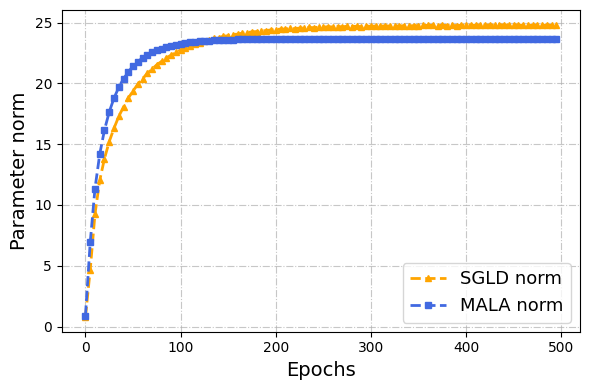}
    \vspace{-0.5em}
    \caption{The training MSE losses of MALA and SGLD are presented on the left (\textbf{left}), with a comparative analysis of their L2 norms shown on the right (\textbf{right}). These results were obtained for $\mathbf{n=800}$, using a learning rate of 0.001 and a parameter dimension of $p=600$ with $\lambda=0.01$, across ten runs. In terms of computational time, SGLD required an average of 12.4 seconds to converge over 200 epochs, while MALA took an average of 12 seconds to converge over 100 epochs.}
    \label{fig:MALA_SGLD_800}
\end{figure*}

In Figure \ref{fig:MALA_SGLD_200} and Figure \ref{fig:MALA_SGLD_800}, when examining both the the training error curve and the convergence rate of the sampling parameter norm over epochs, we observe that both algorithms converge to a similar value. However, it is noteworthy that MALA exhibits a faster convergence in terms of the number of epochs.

When the number of traning samples $n$ is smaller, as in Figure \ref{fig:MALA_SGLD_200}, despite the computational overhead introduced by the additional Metropolis-adjusted step, MALA still outpaces SGLD by approximately 3 seconds. However, for larger sample sizes in Figure~\ref{fig:MALA_SGLD_800}, this Metropolis-adjusted step incurs more additional computational time. Consequently, MALA's convergence time exceeds SGLD's by only 0.4 seconds. Thus, in our experiments, we recommend using MALA for smaller sample sizes, while for larger samples, like when $n=600$ in the Classic regime, SGLD might be more computationally efficient.

\section{Proof of Theorem~\ref{thm:AIC_classical}}\label{appx:Gibbs_AIC}





We note that the original idea of this result is coming from the discussion in Section 4 of~\cite{aminian2021exact}. 

We further assume that the parametric family $\{P(z|\vw), \vw \in \mathcal{W}\}$ and prior $\pi(\vw)$ satisfy all the regularization conditions required for the Bernstein–von-Mises theorem \cite{van2000asymptotic} and the asymptotic Normality of the maximum likelihood estimate (MLE), including Assumption~\ref{assump:MLE} and the condition that $\pi$ is continuous and $\pi>0$ for all $\vw \in \mathcal{W}$.

For the Gibbs algorithm $P^*_{\hat{W}|S}$ defined in \eqref{equ:Gibbs}, if we adopt the log-loss function $\ell(w,\vz) = -\log P(y|\vx;w)$, and set $\beta = n$, the Gibbs algorithm directly corresponds to the Bayesian posterior of the parametric family with prior distribution $\pi$. Therefore,
in the asymptotic regime where $p$ is fixed and $n\to \infty$, Bernstein–von-Mises theorem under model mismatch \cite{van2000asymptotic,kleijn2012bernstein} states that we could approximate the Gibbs algorithm in \eqref{equ:Gibbs} by
\begin{equation*}
   \mathcal{N}(\hat{W}_{\mathrm{ML}}, \frac{1}{n} J(\vw^*)^{-1}),
\end{equation*}
with $\vw^*$ and $J(\vw)$ defined in a similar manner as in \eqref{equ:MLE_normal}.
As $n \to \infty$, the asymptotic Normality of the MLE states that  the distribution of $\hat{W}_{\mathrm{ML}}$ will converge to
\begin{equation*}
    \mathcal{N}(\vw^*, \frac{1}{n} J(\vw^*)^{-1} I(\vw^*)J(\vw^*)^{-1}).
\end{equation*}  
Thus, the Gibbs distribution $P^*_{\hat{W}|S}$ can be approximated as a Gaussian channel regardless of the choice of prior $\pi(\vw)$.
Then, the symmetrized KL information can be computed using~\cite[Theorem 14]{palomar2008lautum}, which characterizes the symmetrized KL information over a vector Gaussian channel, i.e.,  
\begin{equation}
 \frac{I_{\mathrm{SKL}}( P_{\hat{W}|S}^{*}, P_{S})}{n}=\frac{\mathrm{tr}(I(\vw^*)J(\vw^*)^{-1})}{n} =  \mathcal{O}(\frac{p}{n}).
\end{equation}


Therefore, for the Gibbs algorithm $P^*_{\hat{W}|S}$ defined in \eqref{equ:Gibbs}, if we adopt the log-loss function $\ell(w,\vz) = -\log P(y|\vx;w)$, and set $\beta = n$, and let $n \to \infty $ with fixed $p$, we can get the following asymptotic expression for the generalization error,
\begin{equation}
    \frac{I_{\mathrm{SKL}}(P_{\hat{W}|S}^{*},P_S)}{\beta}\rightarrow \frac{p}{n}.
\end{equation}
Plug this result back to \eqref{equ: Gibbs_AIC}, the Gibbs-based AIC can be written as,
\begin{equation}
        \mathrm{AIC^{+}} =  L_{E}(\hat{W}_{\mathrm{Gibbs}},\vz^{n})+ \frac{p}{n}.
\end{equation}


\section{Proof of Proposition~\ref{prop:Gibbs_marginal}}\label{appx:Gibbs_BIC}
If we adopt the log-loss function $\ell(w,\vz) = -\log P(y|\vx;w)$, and set $\beta = n$, the Gibbs distribution in \eqref{equ:Gibbs} can be viewed as the Bayesian posterior distribution, i.e.,
\begin{equation}
    P_{\hat{W}|S}^{*}(\vw|\vz^{n})=\frac{\pi(\vw)\prod\limits_{i=1}^n P(y_i| \vx_i ;\vw)}{V(\vz^{n})}, \text{ with } V(\vz^n)= \int\pi(\vw)\prod\limits_{i=1}^n P(y_i| \vx_i ;\vw) d \vw.    
\end{equation}
Therefore, 
\begin{equation}
\begin{aligned}
 &\mathbb{E}_{P^*_{\hat{W}|S=\vz^n}}\big[L_{E}(\hat{W},\vz^n)\big] + \frac{1}{n}D(P^*_{\hat{W}|S=\vz^n}\| \pi)\\
 &=   \mathbb{E}_{P^*_{\hat{W}|S=\vz^n}}\big[L_{E}(\hat{W},\vz^n)\big] + \frac{1}{n} \mathbb{E}_{P^*_{\hat{W}|S=\vz^n}}\big[\log \frac{\exp\big(-n L_{E}(\hat{W},\vz^n)\big)}{V(z^n)}\big]\\
 &=-\frac{1}{n}\log {V(z^n)}\\
 &= -\frac{1}{n} \log m(\vz^n),
\end{aligned}
\end{equation}
which completes the proof. 

\section{Proof of Theorem~\ref{thm:BIC_Gibbs_classical}}
\label{appx:Gibbs_BIC_classical}

Recall the Gibbs-based BIC is given by
\begin{equation}
    \mathrm{BIC^{+}} \triangleq L_{E}(\hat{W}_{\mathrm{Gibbs}},\vz^n) + \frac{1}{n}D(P^*_{\hat{W}|S=\vz^n}\|\pi).
\end{equation}

Thus, we just need to characterize the asymptotic behavior of the KL divergence term in the regime where $p$ is fixed and $n \to \infty$. Note that The KL divergence can be written as
\begin{equation}\label{equ:KL_entropy}
    \frac{1}{n}D(P^*_{\hat{W}|S=\vz^n}\|\pi) = -\frac{1}{n}H(P^*_{\hat{W}|S=\vz^n}) - \frac{1}{n}\mathbb{E}_{P^*_{\hat{W}|S=\vz^n}}[\log \pi(\hat{W})],
\end{equation}
and $P^*_{\hat{W}|S=\vz^n}$ will converge to $\mathcal{N}(\hat{W}_{\mathrm{ML}}, \frac{1}{n} J(\vw^*)^{-1})$ by Bernstein–von-Mises theorem as shown in Appendix~\ref{appx:Gibbs_AIC}. Then, the second term above will reduce to $\frac{1}{n} \log \pi(\hat{W}_{\mathrm{ML}})$, and converges to zero as $n\to \infty$. 
Using the same Gaussian approximation, the entropy term can be computed as
\begin{equation}
    H(P^*_{\hat{W}|S=\vz^n})= \frac{p}{2}\log(\frac{2\pi e}{n}) +\frac{1}{2}\log|J(\vw^*)^{-1}|.
\end{equation}
As the number of samples $n\to \infty$, $\hat{W}_{\mathrm{Gibbs}} \rightarrow \hat{w}_{\mathrm{ML}}$. Therefore, the Gibbs BIC can be asymptotically approximated by
\begin{align}
   \mathrm{BIC^{+}} &= L_{E}(\hat{W}_{\mathrm{Gibbs}},s)  + \frac{1}{n} D(P^*_{\hat{W}|S=\vz^n}\|\pi)\\
   & = L_{E}(\hat{W}_{\mathrm{Gibbs}},s)  -\frac{1}{n}H(P^*_{\hat{W}|S=\vz^n}) - \frac{1}{n}\mathbb{E}_{P^*_{\hat{W}|S=\vz^n}}[\log \pi(\hat{W})] \\
   &= L_{E}(\hat{w}_{\mathrm{ML}},s)+\frac{p}{2n}\log\frac{n}{2\pi e} + \frac{1}{2n}\log \left|J(\vw^*))\right|-\frac{1}{n} \log \pi(\hat{W}_{\mathrm{ML}}).
\end{align}
Thus, both the terms $p\log 2\pi e$, $\log |J(\vw^*)|$ and $\log \pi (\hat{W}_{\mathrm{ML}})$ have order less than $O(1)$
with respect to the sample size $n$,
and can be ignored in the Gibbs-based BIC.

Comparing the above result with the log-marginal likelihood in~\eqref{equ:BIC_full}, the penalty term in $\mathrm{BIC^{+}}$ differs from the classic BIC by $\frac{p}{2n} \log e$. This is due to the fact that we evaluate the same marginal likelihood $m(\vz^n)$ using different algorithms, and the empirical risk achieved by SGLD or MALA is different from that of the SGD. 

\section{Gibbs Distribution of Random Feature Model}\label{app:RF_posterior}
For random feature model with the prior distribution  $\pi(w) \sim \mathcal{N}(0,\frac{\sigma^{2}}{\lambda n} \mI_p)$ and $L_{E}(w,s)=\frac{1}{n}\sum_{i=1}^{n}\log(y_{i}|x_{i},w)$, the log-posterior $\log (P_{\hat{W}|S}^{*}) \propto \log\pi(w) +\log(e^{-nL_{E}(S)})$ , 
where
\begin{equation*}
   L_{E}(S)=\frac{1}{2n\sigma^{2}}\|\mY-\mB \vw \|_2^2  +\frac{1}{2} \log (2\pi \sigma^{2}).
\end{equation*}
Thus, the Gibbs algorithm, in this case, is given by the following Gaussian posterior distribution, as shown in~\cite{murphy2007conjugate},
\begin{equation}
    P_{\hat{W}|S}^{*} \sim \mathcal{N}(\hat{W}_{\lambda},\mSigma_w),
\end{equation}
where $\hat{W}_{\lambda} = (\lambda n\mI_p+\mB^\top \mB)^{-1} \mB^\top \mY$, and  $ \mSigma_w =\sigma^{2}(\lambda n\mI_p+\mB^\top \mB)^{-1}$.

\section{Proof of Theorem~\ref{thm:over_BIC}}\label{app:RF}
We note that the KL divergence between two Gaussian distributions can be written as 
\begin{equation*}
    D(\mathcal{N}(\vmu_1,\mSigma_1) \| \mathcal{N}(\vmu_2,\mSigma_2)) = \frac{1}{2} \left[ (\vmu_2 - \vmu_1)^\top \mSigma_2^{-1} (\vmu_2 - \vmu_1) +\mathrm{tr}(\mSigma_2^{-1} \mSigma_1)  - p + \log \frac{\det \mSigma_2}{\det \mSigma_1} \right].
\end{equation*}
The KL divergence between the Gibbs posterior of the RF model and the prior can be computed by,
\begin{align}\label{equ:KL_app}
    \frac{1}{n}D(P^*_{\hat{W}|S=\vz^n}\|\pi) &= \frac{1}{2n}\left[\hat{W}_{\lambda}^{\top}(\frac{\sigma^{2}}{\lambda n} \mI_p)^{-1}\hat{W}_{\lambda} + \mathrm{tr}(\frac{\lambda n}{\sigma^{2}}\mSigma_w) +\log \frac{\det(\frac{\sigma^{2}}{\lambda n} \mI_p)}{\det(\mSigma_w)} - p \right] \\
    & = \frac{1}{2n}\left[\frac{\lambda n}{\sigma^{2}}\|\hat{W}_{\lambda}\|_2^2 + \mathrm{tr}\big((\mI_p+\frac{\mB^\top \mB}{\lambda n})^{-1}\big) +\log\det(\mI_p+\frac{\mB^\top \mB}{\lambda n} ) - p \right].\nn
\end{align}

The trace and the log determinant of the random matrix $\mSigma = \frac{\mB^\top \mB}{\lambda n} + \mI_p$ can be computed using the following results from \cite{pennington2017nonlinear}, which characterizes the probability density function of the eigenvalues of the random matrix $\mB^\top\mB/n$ in the over-parameterized regime.

\begin{lemma}\cite{pennington2017nonlinear}\label{lemma:nonlinear}
Let the matrix $\mM = \frac{1}{n}\mB^\top \mB \in \mathbb{R}^{p\times p}$, where $\mB = f\Big(\frac{\mX \mF}{\sqrt{d}}\Big) \in\mathbb{R}^{n \times p}$, all the elements in $\mF \in \mathbb{R}^{d \times p}$ and $\mX \in \mathbb{R}^{n \times d}$ are generated i.i.d from $\mathcal{N}(0,1)$. Suppose that the activation function has zero mean and finite moments, i.e.,
\begin{equation}
    \mathbb{E}[f(\varepsilon)]=0,\quad \mathbb{E}[f(\varepsilon)^k]<\infty,\ \text{ for }  k>1, \quad \varepsilon \sim \mathcal{N}(0,1).
\end{equation}
and define constants $\eta$ and $\xi$ as
\begin{equation}
    \eta \triangleq \mathbb{E}[f(\varepsilon)^2],\quad \xi \triangleq \mathbb{E}[f'(\varepsilon)]^2, \quad \varepsilon \sim \mathcal{N}(0,1),
\end{equation}
as $n,d,p \to \infty$ with $d/p \to \psi$, $d/n \to \phi$, where $\psi,\phi \in (0,\infty)$, then the Stieltjes transform $G(z)$ of the spectral density of random matrix $\mM$ satisfies
\begin{align}
    &dF_\mM(x) = \frac{1}{\pi} \lim_{\epsilon \to 0^+}\mathrm{Im} G(x-i\epsilon),\quad G(z)  = \frac{\psi}{ z}A\Big(\frac{1}{z \psi}\Big) +\frac{1-\psi}{z},\\
    &A(t) = 1+(\eta-\xi)t A_{\phi}(t)A_{\psi}(t) +\frac{A_\phi(t)A_\psi(t)t \xi}{1-A_\phi(t)A_\psi(t)t \xi},
\end{align}
where $A_\phi(t)=1+(A(t)-1)\phi$ and $A_\psi(t)=1+(A(t)-1)\psi$.
\end{lemma}

This lemma characterizes the spectral density of random matrix $\mM$ for any zero-mean activation functions. However, these implicit equations need to be evaluated numerically, and it is hard to obtain a closed-form expression or provide more insights.

If we further assume that the assumptions in (\ref{equ:act-func-conditions}) are satisfied, i.e., $\mathbb{E}[f(\varepsilon)^2]=1$, and $\mathbb{E}[f'(\varepsilon)]^2=0$, then the result in Lemma~\ref{lemma:nonlinear} can be simplified significantly, as the probability density of the  eigenvalues for random matrix $\mM$ will converge to the well-known Marchenko-Pastur distribution
with shape parameter $r=p/n$, i.e.,
\begin{equation}
        dF_\mM(x) \ {\rightarrow}\  \mathcal (1-\frac{1}{r})^{+}\delta(x)+\frac{\sqrt{(x-a)^{+}(b-x)^{+}}}{2\pi r x},
\end{equation}
as $n$, $d$, $p$ all go to infinity, where $(z)^+ \triangleq\max\{0,z\}$, and $a\triangleq(1-\sqrt{r})^2$, and $b \triangleq (1+\sqrt{r} )^2$. Thus, we focus on this case to obtain a  convenient, closed-form expression  for  mutual information.

The following lemma from Sections 2.2.2 and 2.2.3 in \cite{tulino2004random} characterizes the $\eta$-transform and Shannon transform of the Marchenko-Pastur distribution.

\begin{lemma} \label{lemma:shannon}
The $\eta$ and Shannon transform of a nonnegative random variable $X$ are defined as
\begin{equation}
   \eta_X(\gamma) \triangleq \mathbb{E}[\frac{1} {1+\gamma X}],\quad \mathcal{V}_X(\gamma) \triangleq \mathbb{E}[\log (1+\gamma X)],
\end{equation}
respectively. If $X$ is distributed according to Marchenko-Pastur distribution with shape parameter $r=p/n$, then
\begin{align}
   \eta_X(\gamma) &= 1 - \frac{\mathcal{F}(\gamma,r)}{4r\gamma},\\
   \mathcal{V}_X(\gamma) &=  \log \left(1+ \gamma -\frac{1}{4} \mathcal{F}(\gamma,r) \right) + \frac{1}{r}\log \left(1+ \gamma r -\frac{1}{4} \mathcal{F}(\gamma,r) \right) -\frac{1}{4r\gamma}\mathcal{F}(\gamma,r),
\end{align}
where
\begin{equation}
    \mathcal{F}(\gamma,r) \triangleq \left(\sqrt{\gamma(1+\sqrt{r})^2+1} - \sqrt{\gamma(1-\sqrt{r})^2+1}\right)^2.
\end{equation}
\end{lemma}

Equipped with the aforementioned tools from random matrix theory, we could proceed our analysis,
\begin{align}
 \frac{1}{n} \log \det \big(\mI_p + \frac{1}{\lambda n}  \mB^\top \mB \big) 
    = \frac{r}{p} \sum_{i=1}^p\log \big(1+ \frac{1}{\lambda} \lambda_i(\frac{1}{n} \mB^\top \mB) \big),
\end{align}
where the notation $\lambda_i(\cdot)$ denote the eigenvalues of the matrix for $i=1,\cdots,p$. As shown in Lemma~\ref{lemma:shannon},  we have
\begin{equation}
\frac{1}{p} \sum_{i=1}^p\log \big(1+ \frac{1}{\lambda} \lambda_i(\frac{1}{n} \mB^\top \mB) \big) \to  \int_0^\infty  \log \big(1+ \frac{x}{\lambda} ) d F_{\mM}^n(x) = \mathcal{V}_X(1/\lambda)
\end{equation}
almost surely, when $n,d,p \to \infty$, $p/n=r$. Thus, in the over-parameterized regime, we have
\begin{equation}\label{equ:KL_logdet}
   \frac{1}{n}  \log \det \big(\mI_p + \frac{1}{\lambda n}  \mB^\top \mB\big)  \to r\cdot\mathcal{V}_X(1/\lambda) = {V}(1/\lambda, r).
\end{equation}
    
    

And the trace term can be simplified as,
\begin{equation}
\frac{1}{n}\mathrm{tr}(\mI_p + \frac{1}{\lambda n}  \mB^\top \mB \big)^{-1} = r\frac{1}{p}\sum_{i=1}^{p}\frac{1}{\big(1+ \frac{1}{\lambda} \lambda_i(\frac{1}{n} \mB^\top \mB) \big)},
\end{equation}
which will converge to the following expression by Lemma~\ref{lemma:shannon}, when $n,d,p \to \infty$, $p/n=r$,
\begin{equation}\label{equ:KL_trace}
   r\frac{1}{p}\sum_{i=1}^{p}\frac{1}{\big(1+ \frac{1}{\lambda} \lambda_i(\frac{1}{n} \mB^\top \mB) \big)}
   \to r\int_{0}^{\infty} \frac{1}{1+ \frac{x}{\lambda} } d F_{\mM}^n(x)
   =r(1 -\frac{ \mathcal{F}(\frac{1}{\lambda},r) }{4r\frac{1}{\lambda}}).
\end{equation}

Combine \eqref{equ:KL_logdet} and \eqref{equ:KL_trace} with \eqref{equ:KL_app}, we obtain the following result
\begin{equation}
\frac{1}{n}D(P^*_{\hat{W}|S=\vz^n}\|\pi) \to \frac{\lambda}{2\sigma^2} \|\hat{W}_{\lambda}\|_2^2  -\frac{ \lambda }{8} \mathcal{F}(\frac{1}{\lambda},r)+ \frac{1}{2}{V}(1/\lambda, r).    
\end{equation}

\subsection{Empirical Behavior of Covariance Term}\label{sec:exp_cov}




\begin{figure}[!t]
    \centering
    \includegraphics[width=0.325\columnwidth]
    {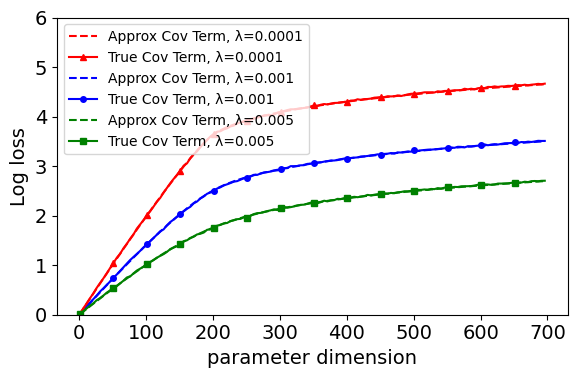} \includegraphics[width=0.325\columnwidth]
    {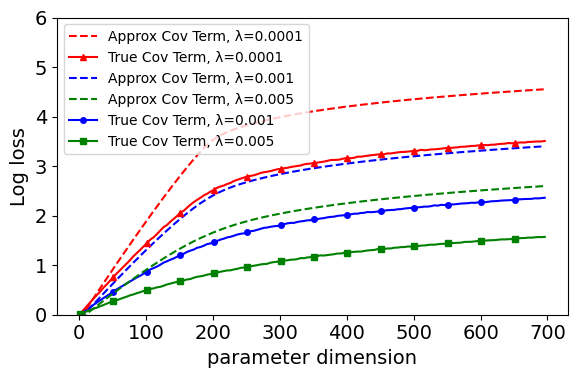}   \includegraphics[width=0.325\columnwidth]{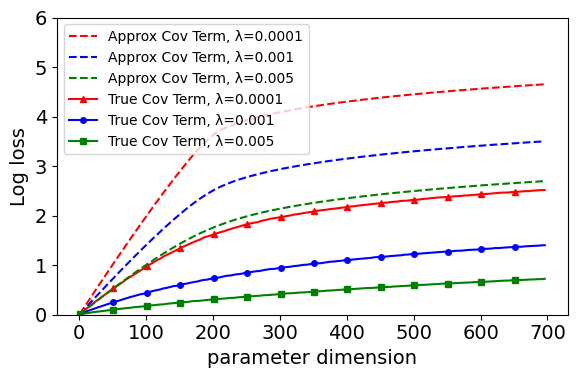}
    \caption{A comparison between $1/n(\log |\mSigma|+\mathrm{tr}(\mSigma)^{-1} -p)$ and the asymptotic approximation of the covariance term in Theorem~\ref{thm:over_BIC} for different value of $\lambda$ and for different activation functions:  $f(x) = (x^2-1)/\sqrt{2}$ (\textbf{left}), ReLU  (\textbf{middle}) and Sigmoid  (\textbf{right}). We adopt the same experiment settings as in Section~\ref{sec:exp_over}, and we change $r =p/n$ by fixing $n=200$ and varying $p$.}
    \label{fig:shannon}
\end{figure}

To show that Theorem~\ref{thm:over_BIC} can provide a good approximation for the asymptotic behavior of the log determinant and trace term in~\eqref{equ:kl_bic},
 we plot in Figure~\ref{fig:shannon}
both the term $1/n(\log |\mSigma|+\mathrm{tr}(\mSigma)^{-1} -p)$ with finite data and $\frac{1}{2}V(1/\lambda, r)-\frac{ \lambda }{8} \mathcal{F}(\frac{1}{\lambda},r)$ in the over-parameterized Gibbs-based BIC for different activation functions with varying regularizer parameters $\lambda$.
As shown from Figure~\ref{fig:shannon}(left),
our theoretical results provide a good proxy for the asymptotic behavior of the covariance term, even for activation functions (e.g., ReLU and Sigmoid in Figure~\ref{fig:shannon}(middle and right))
that do not satisfy the assumptions in Theorem~\ref{thm:over_BIC}.
This is evidence that the particular choice of activation function
does not significantly influence the asymptotic behavior of the covariance term in Gibbs-based BIC.

\section{Additional Experimental Results
}\label{appx:exp}

\subsection{Classic regime}
We instantiate a two-layer RF model  described in~\eqref{equ:RF_model}, where the first layer is designated for feature mapping and is kept random, and we only train the parameter in the second layer.  We use the regularized negative log-likelihood in~\eqref{equ:RF_loss} as the loss function for SGLD to compute both $\mathrm{AIC}^+$ in~\eqref{equ:AIC+classical} and $\mathrm{BIC}^+$ in~\eqref{equ:BIC+classical}, which is different from the SGD algorithm used in classical AIC and BIC. 






For simplicity, we generate $n = 600$ training samples from this linear model 
\begin{equation}\label{equ:linear_truth}
     y_i = \vx_i^\top \vw^* +\epsilon_i,\quad \vw^*\in \mathbb{R}^p, \quad \|\vw^*\|_2^2=1, \quad\epsilon_i \sim \mathcal{N}(0, \sigma^2),
\end{equation}
with $p=80$, and noise $\sigma^2=0.2$.


In Figure~\ref{fig:under_loss} (Left), we plot the training and test MSEs achieved using SGD and SGLD, respectively. Although SGD and SGLD might yield different model parameters, they exhibit a similar trend in terms of MSE and Log-Loss (see Appendix~\ref{appx:exp}). The discrepancy in the training algorithms does not significantly impact model selection. Figure~\ref{fig:under_loss} (middle) demonstrates that both AIC and $\mathrm{AIC}^+$ follow a similar trend to the test MSE, selecting the model ($p=110$) that achieves the smallest population risk. On the other hand, due to larger penalty terms, BIC and $\mathrm{BIC}^+$ favor simpler models and identify the model most likely to generate the training data. As shown in Figure~\ref{fig:under_loss} (right), the classical BIC selects the model with $p=70$, while $\mathrm{BIC}^+$ favors the model with $p=80$.
In summary, our observations suggest that $\mathrm{AIC}^+$ and $\mathrm{BIC}^+$ based on SGLD exhibit similar model selection capabilities to the traditional AIC and BIC in the classical large $n$ regime.
\begin{figure}[!t]
    \centering
   \includegraphics[width=0.32\columnwidth]{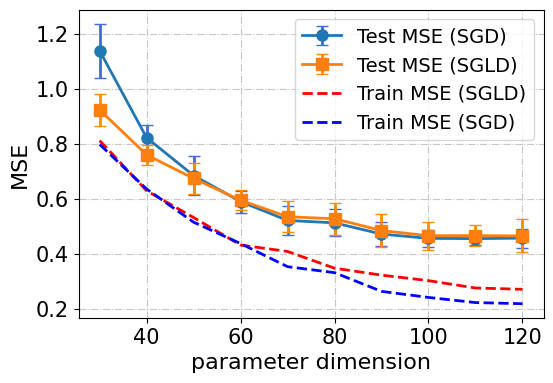}
\includegraphics[width=0.32\columnwidth]{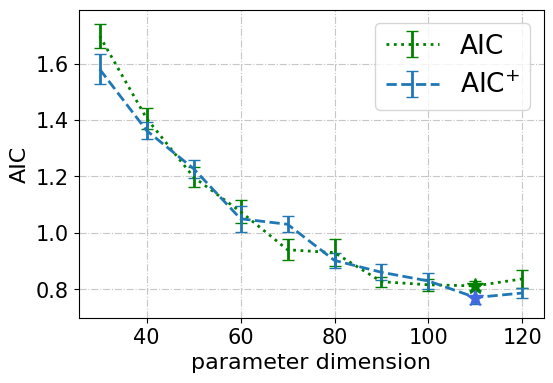}    \includegraphics[width=0.32\columnwidth]{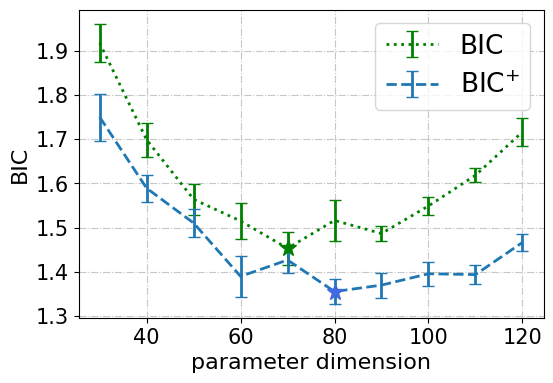}
   
    \caption{
        A comparison of SGD and SGLD in terms of MSE  (\textbf{left}). Comparisons of the classical AIC with $\mathrm{AIC}^+$ in \eqref{equ:AIC+classical} (\textbf{middle}), and the classical BIC with $\mathrm{BIC}^+$ in \eqref{equ:BIC+classical} (\textbf{right}). All experiments are conducted in the classical $n\gg p$ setting, with $n=600$, $p=[30,...,120]$. The preferred models selected by different information criteria are marked using stars with different colors.        
    }
\label{fig:under_loss}
\end{figure}

\subsection{Experimental Details}

In our experiments, the setup for the random feature model is akin to a two-layer neural network with a ReLU activation function. The first layer is initialized with a Gaussian distribution $\mathcal{N}(0,\frac{1}{\sqrt{d}})$ and remains unchanged throughout the training. The second layer, which starts off with zero values, is trained using the log loss function.

We set the step size for both SGD and MALA to $0.01$. The training process of MALA typically converges around  100 epochs. However, to ensure convergence to the Gibbs distribution, we continue to run MALA for a substantial number of epochs (600 in our experiment) even after achieving the minimum training loss. In addition, although we trained the noise variance in the classic regime to calibrate it to an appropriate scale, we fixed the noise variance $\sigma$ to 0.05 due to the high variability in model complexity in the over-parameterized regime. 

Our experiments, implemented in 
PyTorch requires less than 24 hours of training computation time on a single RTX 3090 GPU. To ensure result accuracy and plot error bars, we perform 50 runs for each setting.

\subsection{RF model loss}
Keeping remind that the following loss function is only used in MALA training process,
\begin{align}
    \mathcal{L}(\vw) 
     &=-\sum_{i=1}^n \Big( \log P(y_{i}|x_{i},w) -\frac{1}{n}\log \pi(w) \Big)\\
    &=\frac{1}{2\sigma^{2}} \sum_{i=1}^n \left(y_i - f\left(\frac{\vx_i^\top \mF}{\sqrt{d}}\right) \vw \right)^2 +\frac{n}{2} \log (2\pi \sigma^{2}) +\frac{\lambda n \|\vw\|_2^2}{2\sigma^{2}}.   
\end{align}
For traditional AIC and BIC, we optimize the following log-likelihood using SGD,
\begin{equation}
\mathcal{L}(\vw) 
=\frac{1}{2\sigma^{2}} \sum_{i=1}^n \left(y_i - f\left(\frac{\vx_i^\top \mF}{\sqrt{d}}\right) \vw \right)^2 +\frac{n}{2} \log (2\pi \sigma^{2}).
\end{equation}

\end{document}